\documentclass[9pt,twocolumn,twoside]{pnas-new}

\templatetype{pnasresearcharticle} 
\usepackage{fontawesome}
\usepackage{lipsum}
\usepackage{textcomp}
\usepackage{graphicx}
\usepackage{multirow}
\usepackage{lipsum}
\usepackage{textcomp}
\usepackage{natbib}
\usepackage{fontawesome}
\usepackage{bbm}\usepackage{listings}
\usepackage{textgreek} 
\usepackage{amsmath} 
\usepackage{amssymb,amsthm,mathtools}
\usepackage[capitalise]{cleveref}
\usepackage{stfloats} 
\usepackage[normalem]{ulem}

\newtheorem{theorem}{Theorem}[]

\title{PCS-UQ: Uncertainty Quantification via the Predictability-Computability-Stability Framework}

\author[a,1]{Abhineet Agarwal}
\author[a,1]{Fange Xiao}
\author[c,1]{Rebecca Barter}
\author[a]{Omer Ronen}
\author[a]{Boyu Fan}
\author[a,b,2]{Bin Yu}

\affil[a]{Department of Statistics, University of California, Berkeley}
\affil[b]{Department of Electrical Engineering and Computer Science, University of California, Berkeley}
\affil[c]{Department of Epidemiology, University of Utah}

\leadauthor{Agarwal}

\significancestatement{
Trustworthy uncertainty quantification is key for responsible data science and AI (including data-driven decision-making). 
Traditional statistical inference methods require specifying an underlying generative model and are often not robust to model misspecification. 
Here, we propose an uncertainty quantification method based on the predictability-computability-stability framework for veridical data science to produce accurate and robust prediction intervals. 
Our PCS-driven prediction intervals achieve the desired coverage across a wide variety of experiments, while improving interval width relative to oracle-selected conformal inference approaches to varying degrees and achieving more consistent subgroup coverage.
}

\authorcontributions{}
\authordeclaration{The authors declare no conflicts of interest.}
\equalauthors{\textsuperscript{1} A.A., F.X., and R.B. contributed equally to this work.}
\correspondingauthor{\textsuperscript{2}To whom correspondence should be addressed. E-mail: binyu@berkeley.edu}


\DeclareUnicodeCharacter{221E}{\ensuremath{\infty}}

\begin{document}
\begin{abstract}
As machine learning (ML) enters high-stakes domains, trustworthy uncertainty quantification (UQ) is essential for safety. 
In this paper we introduce PCS-UQ, a framework based on the Predictability, Computability, and Stability (PCS) principles for veridical data science. 
Starting with a candidate set of models or algorithms, PCS-UQ integrates a rigorous prediction-check to screen out unsuitable models in the set and utilizes bootstrap samples in order to capture both inter-sample variability and algorithmic instability for the prediction-checked algorithms. 
We then introduce a novel multiplicative calibration scheme to enhance local adaptivity, which can be viewed as a new score in conformal prediction. 
Moreover, we produce a compilation of 17 real-world regression datasets with manually constructed subgroups. 
On this benchmark, PCS-UQ maintains the target coverage while outperforming or matching conformal methods equipped with oracle-selected algorithms in interval width. 
PCS-UQ achieves consistent subgroup coverage, outperforming these oracle-selected conformal methods. 
Notably, PCS-UQ stands out in achieving both competitive interval widths and consistent subgroup coverage.
Across 6 classification datasets, PCS-UQ reduces prediction set sizes by 20\%. 
To scale the framework for deep learning, we propose computationally efficient variants that bypass expensive retraining. 
On three computer vision benchmarks, these variants reduce prediction set sizes by 20\% over conformal baselines. 
Finally, we provide a theoretical proof that a modified PCS-UQ algorithm preserves valid coverage under exchangeability as a form of split conformal inference.
\end{abstract}
 
\maketitle
\ifthenelse{\boolean{shortarticle}}{\ifthenelse{\boolean{singlecolumn}}{\abscontentformatted}{\abscontent}}{}

\section{Introduction}
\label{sec:intro}
Recent decades have seen tremendous growth in machine learning (ML) and artificial intelligence (AI). 
As these systems increasingly inform high-stakes decisions, ensuring their reliability and safety has become a central concern.
Failures in trust and reproducibility---exemplified by the replication crisis in biomedical research \cite{ioannidis2005most,begley2012raise,open2015estimating}---highlight the risks of reaching conclusions based on questionable models.
A key component of establishing confidence and enabling responsible decision-making from data and models is trustworthy uncertainty quantification (UQ). 
Accurate estimates of uncertainty allow practitioners to reach reliable data-driven conclusions and accurately assess risk to mitigate downstream consequences.
Indeed, researchers believe that poor estimates of uncertainty were a significant cause of the biomedical replication crisis \cite{ioannidis2005most}.

Standard approaches to UQ are based on a traditional statistical modeling framework pioneered by R.A. Fisher and others a century ago.
This framework relies on specifying a probabilistic generative (i.e., ``true'') model whose parameters we estimate via observed data.
While this framework provides tractable mathematical models to analyze \cite{cox2006principles,reid2015}, it was not designed for the complexities of modern ML models and datasets. 
For example, large language models (LLMs) \cite{vaswani2017attention, radford2018improving} consist of hundreds of billions of parameters and are trained on web-scale datasets that consist of various modalities, e.g., tables, text, and images. 
In these settings, simple generative models and the assumptions required for valid statistical inference are unlikely to hold. 
These limitations call for new tools to quantify uncertainty to support responsible, data-driven decision-making.

To move beyond reliance on correctly specified generative models, statisticians have developed conformal inference \cite{vovk2005algorithmic,shafer2007tutorialconformalprediction,lei2018distribution}.
Assuming exchangeable data, conformal inference is a distribution-free framework that produces valid prediction sets. 
Conformal inference has been the subject of intense study over the past decade, leading to many impressive theoretical results and practical extensions \cite{angelopoulos2021gentle}.
One highlight of the framework is its validity guarantee, which holds for any predictive model, regardless of quality.
However, an inaccurate model can produce unnecessarily large and often unstable prediction sets.
In practice, the task of model checking is left to practitioners and often overlooked.
Additionally, because the validity guarantee is marginal, conformal methods can also exhibit coverage shortfalls on subgroups (examined in \cref{sec:reg_exp}) --- a key concern for practitioners.
Lastly, despite the breadth of theoretical work, empirical comparisons of conformal methods remain limited, leaving open the question of how these practical challenges play out across varying data conditions.

While conformal inference has substantially advanced distribution-free prediction, robust data analysis will ultimately require a broader view of UQ.
It will require considering uncertainty in every stage of the data science life cycle (DSLC), from problem formulation and data collection to exploratory analyses, data cleaning, modeling, interpretation, and even visualization \cite{yu2020veridical}.  
At each stage, researchers face choices --- data cleaning methods, model or algorithm choices, hyperparameter tuning, and more --- that can have a large and often unacknowledged effect on analyses, results, and conclusions.
For example, Breznau et al. \cite{breznau2022observing} showed that even when given the same dataset and the same domain problem, different teams of social scientists made choices that led to opposite conclusions. 
More broadly, the vast number of choices available to researchers throughout the DSLC \footnote{Gelman and Loken \cite{gelman2013garden} refer to this as the ``garden of forking paths''} creates a hidden universe of uncertainty that is often ignored \cite{simmons2011false, gelman2013garden}.

To address this challenge, Yu and Kumbier \cite{yu2020veridical} proposed the Predictability-Computability-Stability (PCS) framework for veridical data science, which recognizes that data-driven conclusions are the result of multiple steps and human judgment calls. 
The PCS framework provides a philosophical and practically structured approach to guide these choices by unifying, streamlining, and expanding on the ideas and best practices of statistics and machine learning:
First, the PCS framework requires that the ML models used are \textit{predictive} (i.e., prediction-checked or screened) for each stage of a DSLC.
Second, PCS formally considers computation both in terms of time/memory complexity and the use of data-inspired simulations to further augment data analyses \cite{elliott2025designing}.
Third, the stability principle \cite{yu2013stability} considers reasonable perturbations and choices made in the DSLC; it both assesses instability of conclusions relative to these perturbations and appropriately aggregates prediction-checked (or reality-checked) models (steps more broadly) for better performance.
Moreover, PCS requires meticulous documentation of the reasonable choices made in a DSLC or an AI workflow.

PCS has empirically been shown to be effective across a range of challenging domain problems from developmental biology \cite{wu2016stability}, genomics \cite{basu2018iterative}, stress-testing clinical decision rules \cite{basu2018iterative}, subgroup discovery in causal inference \cite{dwivedi2020stable}, and cardiology \cite{wang2023epistasis}.
More recently, Yu and Barter detail the PCS framework in their textbook on veridical data science \cite{yu2024veridical}.
In \cite{yu2024veridical}, the ``P'' in PCS is extended to stand for \emph{reality check} in every step of a DSLC; the book also extends the framework beyond supervised learning to unsupervised learning (see also \cite{rewolinski2025pcsworkflow} for a further extension to building trustworthy AI workflows). 
That is, PCS uses (broadly interpreted) predictability as a proxy for reality check to ensure that every step of a DSLC captures reality.
Yu and Barter turn to the specific problem of UQ in chapter 13, proposing an initial PCS-driven UQ (PCS-UQ) method for regression built from first PCS principles (and including uncertainty considerations from both data cleaning and model choices) and empirical evidence.
Although developed independently from conformal inference, aspects of PCS-UQ, such as using bootstrap resampling for stability and adopting novel calibration strategies, converge with ideas that have recently gained traction in conformal literature \cite{kim2020predictive, qiu2023prediction}.
This convergence lends support to the robustness of these methodological choices.
The \textbf{goal} of this paper is to build on prior work and further develop a \emph{PCS-driven} UQ method that (1) maintains validity guarantees and (2) addresses additional objectives (e.g., model checking, subgroup coverage) beyond those of existing UQ methods.
To evaluate PCS-UQ and conformal inference methods, we gathered 17 benchmark datasets from public sources.
We focus on uncertainty arising from inter-sample variability and place explicit emphasis on rigorous model checking for both PCS-UQ and conformal methods, and leave the incorporation of data-cleaning choices and other human judgment calls to future work.
Our contributions are as follows.

\paragraph{PCS-UQ for Regression} As previously alluded to, we build upon a recently proposed PCS-UQ method for regression in chapter 13 of \cite{yu2024veridical}. 
For simplicity, we provide a summary of the original method here, and describe the procedure used in the paper in \cref{sec:PCS_reg}.
We are given a set of candidate algorithms or models. (1) We split data into training and validation sets, and fit the given candidate prediction algorithms on the training set. In accordance with the ``P'' principle, we drop candidate algorithms that perform poorly on the validation set. 
(2) Next, following the stability principle, we fit the filtered or prediction-checked set of algorithms on multiple bootstrapped training datasets. 
These discrete sets of bootstraps create a \emph{pseudo-population} that allows us to assess finite-sample uncertainty. 
(3) Lastly, we perform a multiplicative calibration that extends interval lengths to achieve the desired coverage (see \cref{supp:PCS_ch13} for details).
We perform extensive experiments across 17 real-world regression datasets. 
Results show that PCS-UQ achieves the desired marginal coverage while reducing or matching the length of intervals relative to leading ``oracle-selected'' \footnote{As seen later in the paper, we select conformal methods with the best average test-set performance across datasets.} conformal approaches.
Further, we show that our multiplicative calibration approach allows PCS-UQ to achieve the desired coverage across subgroups, whereas conformal inference approaches do not consistently do so.
As we discuss later, multiplicative calibration can be regarded as a novel conformal ``score'' function; this directly connects PCS-UQ to conformal prediction.

\paragraph{Real-World Dataset Compilation} To our knowledge, existing empirical comparisons of conformal methods have been conducted on up to 5 real-world datasets. 
We produce a compilation of 17 real-world datasets from public sources for regression tasks.
Beyond aggregate performance, we construct natural subgroups for each regression dataset based on natural breaks in a feature's distribution. 
This allows a form of local analysis that is critical for practitioners but rarely reported in the conformal literature.

\paragraph{PCS-UQ for Classification} We extend PCS-UQ from regression to multi-class classification. 
Experiments across $6$ multi-class datasets show improvement upon ``oracle-selected'' conformal approaches by over $20\%$ in prediction set size.

\paragraph{Approximation Methods for Deep-Learning} PCS-UQ for regression and classification requires fitting multiple models across bootstraps, which can be prohibitively expensive for large deep-learning (DL) models. 
We propose two approximation methods to avoid fitting multiple DL models.
Specifically, we only train one model and either apply dropout on activations \cite{gal2016dropout} or add Gaussian noise to the weights \cite{zhou2024knowgraph, gan2026neural} to create multiple perturbed models. 
Experiments across three computer vision benchmarks show that our approximation schemes maintain the computational efficiency of conformal inference while achieving valid coverage and reducing prediction set sizes by $20\%$.

\paragraph{Connection to Conformal Inference} We show that the multiplicative calibration approach in PCS-UQ can be regarded as a novel conformal ``score'' function. 
Under an exchangeability assumption and a slightly modified PCS-UQ algorithm, we show that this modified approach achieves the desired coverage, providing a formal bridge between PCS-UQ and split conformal.
\section{Related Work}
\label{sec:related_work}

\paragraph{Classical Parametric Inference} As discussed, classical statistical approaches consider uncertainty under a fixed generative, often linear model \cite{cox2006principles, reid2015}. 
Typically, these approaches focus on deriving analytic distributions of parameter estimators \cite{belloni2012sparse, buhlmann2013, zhang2014confidence, vandeGeer2014,javanmard2014confidenceintervalshypothesistesting}.
Another significant line of work is post-selection inference, which focuses on statistical inference in the best linear approximation of an underlying regression function \cite{fithian2014optimal,tibshirani2016exact,lee2016exact,tian2017asymptotics}. 
These methods, while influential, are not the focus of our work since they specify an underlying generative model and focus on theoretically studying the confidence intervals for parameters. 
In contrast, our work aims to empirically construct and evaluate trustworthy \emph{prediction intervals} without assuming such a model.

\paragraph{Resampling} Resampling to assess uncertainty has been widely studied in statistics. 
Prominent among resampling methods are the bootstrap \cite{efron1992bootstrap,stine1985bootstrap}, subsampling \cite{politis1994large, bickel1997resampling}, and the jackknife \cite{quenouille1949approximate,quinlan1986induction}.
The bootstrap is a key component of PCS-UQ since we use it to assess finite-sample variability for our screened models.  %
There have also been a number of related papers that use leave-one-out approaches for constructing prediction intervals  \cite{stone1974cross, butler1980predictive}. 
These methods typically do not address model checking or perform model screening, and require re-fitting the model for every training sample, which renders them infeasible for modern ML models. 
See Efron and Gong \cite{efron1983leisurely} for a comprehensive overview of different approaches. 
There has also been a line of work to quantify uncertainty of ensemble methods based on resampling methods such as bagging and bootstrapping, e.g., Random Forests\cite{mentch2016quantifying, wager2014confidence}.

\paragraph{Conformal Inference for Regression} Proposed by Vovk \cite{vovk2005algorithmic, shafer2007tutorialconformalprediction}, conformal prediction for regression has been a major focus of theoretical study. 
If the underlying data are exchangeable, conformal methods achieve target coverage. 
Split conformal prediction \cite{papadopoulos2002inductive,lei2018distribution} is the 
most widely used form of conformal inference, and is based on a simple and effective idea. 
First, split the data into two halves, using one half for fitting a model, and the other to calibrate prediction intervals to achieve the desired coverage. 
Recent work \cite{barber2021predictive,kim2020predictive} has also combined resampling techniques such as the jackknife and bootstrap with conformal inference to reduce interval lengths. 
The works discussed above achieve desired coverage on average, but there are no guarantees for local coverage, i.e., conditional on covariates or for subgroups.
As a result, different methods such as Studentized conformal inference \cite{lei2018distribution} and kernel-weighted conformal methods have been proposed to improve local coverage \cite{guan2020conformalpredictionlocalization}.
Other extensions include techniques to tackle covariate shift \cite{tibshirani2019conformal}, time-series \cite{stankeviciute2021conformal,angelopoulos2023conformal}, and treatment effect estimation in causal inference \cite{lei2021conformal}. 
Since the conformal literature is too broad to cover comprehensively, we refer readers to \cite{shafer2007tutorialconformalprediction,angelopoulos2021gentle} for a detailed overview.  

\paragraph{Conformal Classification} Romano et al. \cite{romano2020classification} proposed a new conformal score function for categorical and ordinal responses.
Specifically, they propose an approach called adaptive prediction sets (APS) which is based on a cumulative likelihood score. 
For a given sample, APS creates prediction sets by greedily adding classes in order of the predicted probability until the cumulative score of the set reaches a threshold. 
This threshold is calibrated to achieve the desired coverage. 
Angelopoulos et al. define a regularized version of APS called RAPS that has been shown to improve set size in practice \cite{angelopoulos2022uncertaintysetsimageclassifiers}.

\section{PCS Regression Prediction Intervals}
\label{sec:PCS_reg}
We detail the PCS-UQ procedure for generating prediction intervals in the regression setting. 
Our method is closely related to and builds on the procedure proposed in chapter 13 of \cite{yu2024veridical}; see \cref{supp:PCS_ch13} for an overview of this method.  
Extensions to multi-class classification are discussed in \cref{sec:pcs_mc}.
This paper does not focus on uncertainty generated by data cleaning choices, and instead focuses on uncertainty resulting from label noise and finite samples. 
Before detailing our algorithm, we establish necessary notation. 

\paragraph{Notation} We work in the typical supervised regression setting with data $\mathcal{D} = \{(\mathbf{X}_i,Y_i)\}^n_{i=1}$, where $\mathbf{X}_i \in \mathbb{R}^d$, and $Y_i \in \mathbb{R}$. 
For $\alpha \in (0,1)$, the goal is to produce intervals that achieve $1-\alpha$ coverage. 
That is, we aim to produce prediction intervals that contain the true response for $1 - \alpha$ proportion of future data points.
Let $f_1 \ldots f_{M}$ denote candidate predictive algorithms, e.g., ordinary least squares (OLS), Random Forests (RFs), etc. 
Finally, let $l$ denote a loss, e.g., mean-squared error. 

\paragraph{Step 1: Data-Splitting and Prediction-Check}
Randomly split $\mathcal{D}$ into a training set $\mathcal{D}_{\text{tr}}$, and a validation set $\mathcal{D}_{\text{val}}$. Train each algorithm on the training set to obtain fitted models $\hat{f}_1(\cdot; \mathcal{D}_{\text{tr}}), \ldots, \hat{f}_M(\cdot; \mathcal{D}_{\text{tr}})$. Choose the top-$k$ performing algorithms according to loss $l$ \footnote{Chapter 13 of \cite{yu2024veridical} describes other data-driven ways to perform model screening.}.  Without loss of generality, let $f_1 \ldots f_k$ denote the top-$k$ performing algorithms. 
The number of algorithms to include, $k$, serves as a hyperparameter in PCS-UQ; we discuss data-driven choices for $k$ later. 

\paragraph{Step 2: Bootstrapping} Bootstrap the \emph{entire} dataset $B$ times to obtain bootstrapped samples $\mathcal{D}^{(1)} \ldots \mathcal{D}^{(B)}$. Train all algorithms chosen in the previous step on every bootstrapped dataset $\mathcal{D}^{(b)}$ to obtain bootstrapped models $\{\hat{f}_j(\cdot; \mathcal{D}^{(b)}), j \in [k], b \in [B]\}$. 
For each $(\mathbf{X}_i, Y_i) \in \mathcal{D}$, let $T_i \subseteq [B]$ be the set of bootstrap indices such that $(\mathbf{X}_i,Y_i) \notin \mathcal{D}^{(b)}$ for all $b \in T_i$ \footnote{We use out-of-bag (OOB) samples to replace a fixed validation set, which is used in the PCS-UQ method introduced in Chapter 13 of \cite{yu2024veridical}.}.

\paragraph{Step 3: Calibration} First, for each $(\mathbf{X}_i, Y_i)$, form a prediction set $\mathcal{P}_i = \{\hat{f}_j(\mathbf{X}_i; \mathcal{D}^{(b)}); j \in [k], b \in T_i\}$.
Then, form an uncalibrated interval $[q_{\alpha/2}(\mathcal{P}_i), q_{1 - \alpha/2}(\mathcal{P}_i)]$, where $q_{\beta}(S)$ is the $\beta$ quantile for a set $S$. 
For a multiplicative scaling factor $\gamma$, generate a scaled interval
\begin{equation*}
\begin{aligned}
    \mathcal{I}_i (\gamma)  = &\Big[q_{0.5} ({\cal P}_i) - \gamma \times 
    \big(q_{0.5} ({\cal P}_i) - q_{\alpha/2} ({\cal P}_i)\big), \\
    &\quad q_{0.5} ({\cal P}_i) + \gamma \times 
    \big(q_{1-\alpha/2} ({\cal P}_i) - q_{0.5} ({\cal P}_i)\big) \Big].
\end{aligned}
\end{equation*}
We choose the scaling factor $\hat{\gamma}$ such that $ \frac{1}{n} \sum_i I_{\{ Y_i \in \mathcal{I}_i (\hat{\gamma}) \}} \geq 1- \alpha$, i.e., we achieve $1-\alpha$ coverage on the data $\mathcal{D}$.

\paragraph{Step 4: Generating PCS Prediction Interval for Test-Point} For a new test point $\mathbf{X}$, let $\mathcal{P} = \{\hat{f}_j(\mathbf{X}; \mathcal{D}^{(b)}); j \in [k], b \in [B]\}$. Then, we produce prediction interval 
\begin{equation}
\label{eq:pcs_reg_interval}
\begin{aligned}
    \mathcal{I} = &\Big[q_{0.5} ({\cal P}) - \gamma \times 
    \big(q_{0.5} ({\cal P}) - q_{\alpha/2} ({\cal P})\big), \\
    &\quad q_{0.5} ({\cal P}) + \gamma \times 
    \big(q_{1-\alpha/2} ({\cal P}) - q_{0.5} ({\cal P})\big) \Big].
\end{aligned}
\end{equation}

The PCS-UQ algorithm consists of four key steps that contribute to its strong performance; see \cref{sec:reg_exp}.
 We discuss the motivation behind these design choices and how they compare to common conformal methods.

\begin{enumerate}[leftmargin=*]
    \item \textbf{Prediction-check.} PCS incorporates an explicit model-checking step that screens out algorithms with poor prediction performance.
    This is done in order to ensure that uncertainty is assessed using algorithms that sufficiently capture the underlying data-generating process. 
    Experiments in \cref{sec:ablation} show that excluding models with poor predictive performance leads to significantly smaller intervals.   
    In general, conformal methods do not explicitly include prediction-check.
    However, in our experiment, we perform a model-checking or ``oracle-selection'' step on conformal by presenting results using particular algorithms with the best average test set performance (see \cref{sec:reg_exp}).
    
    \item \textbf{Assessing local uncertainty via bootstraps} To assess uncertainty from finite samples, PCS simulates the data-collection process by constructing a set of perturbed datasets via the bootstrap. 
    This universe of discrete datasets creates a \emph{pseudo-population} that allows us to quantify \emph{local} uncertainty.
    Specifically, by evaluating the ensemble of bootstrap models (after prediction-check) at a given data point $X$, we construct an empirical conditional distribution of the predictions. 
    The quantiles of this distribution characterize the local spread of the models at $X$, providing a data-driven measure of uncertainty for the specific sample.
    On the other hand, split-conformal methods only utilize one random train-calibration split, which, while computationally efficient, can introduce variability across different splits.
    Bootstrap-based conformal methods \cite{kim2020predictive} mitigate this by leveraging multiple resampled datasets (without model checking or prediction-check);
    however, they aggregate bootstrap predictions into a point estimate rather than preserving the full predictive distribution, limiting their ability to characterize local uncertainty.
    Lastly, we note that using the bootstrap does impose a higher computational cost, especially on large datasets. We conduct a detailed analysis and discuss methods to reduce computation time in \cref{supp:reg_results}.
    
    \item \textbf{Data-efficiency via out-of-bag samples} Traditional split-conformal methods and the PCS method proposed in \cite{yu2024veridical} use a data split that leaves less data available for both fitting models and calibration. 
    In this paper, we use OOB samples to use the data efficiently as described in Step 2 above. 
    Results in \cref{supp:reg_comparison_pcs} show that using OOB samples reduces interval length by $\approx 5\%$ on average.  
    
    \item \textbf{Multiplicative calibration} Instead of additive calibration (i.e., expanding intervals by a fixed constant) as is common in conformal inference, we calibrate multiplicatively. 
    Since additive calibration expands intervals for every sample by a fixed length, it does not adjust the interval according to how uncertain the model's prediction is for that sample. 
    While the Studentized conformal method addresses this through explicit residual modeling, PCS-UQ achieves local adaptivity through a multiplicative scaling factor that naturally widens intervals for samples with high uncertainty.
    Experiments in \cref{sec:ablation} show that replacing multiplicative with additive calibration in the PCS procedure leads to poorer subgroup coverage across datasets. 
\end{enumerate}

\paragraph{Hyperparameter choices} PCS-UQ depends on two hyperparameters, $k$ and $B$.
These hyperparameters are selected using synthetic simulations and 5 pilot datasets.
To avoid contamination, we do not include these datasets in our 17-dataset benchmark.
For $k$, we choose the number of algorithms that leads to the smallest width while maintaining the desired coverage in synthetic simulations and 5 pilot datasets.
This results in our choice of $k=1$.
Note that we set a fixed number of algorithms for simplicity.
In practice, the specific configuration of prediction-check should be decided by context and domain knowledge as suggested in Ch. 13 of \cite{yu2024veridical}.
We choose $B=1000$ to be as large as computationally feasible.

\section{Regression Experiments}
\label{sec:reg_exp}

\subsection{Experimental Setup}
\label{subsec:reg_exp_setup}

This section details the experimental set-up for our regression experiments displayed below. 

\paragraph{Datasets} We use $17$ regression datasets commonly found in tabular benchmarks \cite{matthias2021openml}. 
These datasets reflect a range of sample sizes and dimensions. 
To avoid uncertainty associated with data cleaning, our datasets do not contain any missing values.  
We use $80\%$ to train and fit various UQ methods, and reserve $20\%$ for the test set.
For methods requiring a further split of the training set (for training algorithms and calibration), we use an even split of the training set, following \cite{lei2018distribution}.
Results using a 75/25 split are similar and can be found in \cref{supp:reg_different_split}.

\begin{table}[htbp]
    \centering
    \begin{tabular}{clrr}
    \toprule
    & Name &  Samples &  Features\\
    \midrule
        & Energy \cite{tsanas2012energy} &    768 &         10 \\
        & Concrete \cite{yeh1998concrete} &      1030 &        8 \\
        & Insurance \cite{ali2020pycaret} & 1338 & 8 \\
        & Airfoil \cite{brooks1989airfoil} &      1503 &         5 \\
        & Debutanizer \cite{fortuna2007soft} &     2394  &         7 \\
        & QSAR \cite{olier2015qsar} &      5742 &         500 \\
        & Parkinsons \cite{tsana2009parkinsons} &      5875 &        18 \\
        & Kin8nm \cite{torgo2014kinematics} &      8192 &         8 \\
        & Computer \cite{torgo2014computer} &     8192 &        21 \\
        & Powerplant \cite{tfekci2014combined} &      9568 &         4 \\
        & Naval \cite{coraddu2015condition} &    11934 &         24 \\
        & Miami \cite{bourassa2021miami} &     13932 &         28 \\
        & Elevator \cite{torgo2014elevators} &      16599 &        18 \\
        & CA Housing \cite{kelley1997sparse} &      20640 &         8 \\
        & Superconductor \cite{hamidieh2018superconductivity} &      21263 &         79 \\
        & Protein \cite{rana2013physicochemical} &    45730   &         9 \\
        & Diamond  \cite{matthias2021openml} &      53940 &        23 \\    
     \bottomrule
     \end{tabular}
    \caption{Datasets used for regression experiments.}
    \label{tab:datasets}
\end{table}

\paragraph{Baseline Conformal Methods} We compare PCS against three popular conformal regression methods: split conformal regression \cite{harris2002inductive, lei2018distribution}, Studentized conformal regression \cite{harris2002inductive, lei2018distribution}, and jackknife+-after-bootstrap (J+aB) \cite{kim2020predictive} \footnote{We previously compared against the Majority Vote procedure from \cite{gasparin2024merginguncertaintysetsmajority} but removed the result for simplicity. Archived results can be found on our GitHub \url{https://github.com/aagarwal1996/PCS_UQ}.}.
For each run on a dataset and an ML method, split and Studentized conformal each trains a single model, while J+aB trains $B$ bootstrapped models.
We use the following eight candidate ML models: Ordinary Least Squares (OLS), Ridge regression \cite{hoerl1970ridge}, Lasso \cite{tibshirani1996regression}, Elastic Net \cite{zou2005regularization}, Random Forests \cite{breiman2001random}, AdaBoost \cite{freund1997decision}, XGBoost \cite{chen2016xgboost}, and a one-hidden-layer multilayer perceptron (MLP). 
We choose regularization parameters in Ridge, Lasso, and Elastic Net via three-fold cross-validation.
For other ML models, we use the default hyperparameters from \texttt{scikit-learn} \cite{scikit-learn}.
Additionally, we create a bagged ensemble with the top three of the ML models selected via a small (10\% of the training set) validation set. 

\paragraph{Metrics} We measure coverage and width of intervals on the test set.
We aim for $90\%$ coverage, i.e., we set $\alpha = 0.1$.
Interval width is normalized by the range of the responses on the test set. 
Results are averaged across 10 train-test splits. 

\subsection{Results}  
\label{subsec:regression_results}

This section reports the empirical results for our experiments described in the previous section. 

\paragraph{A note on ``oracle-selection'' for baselines} Due to the large number of conformal methods we ran, we only report results for the split conformal trained with XGBoost, Studentized conformal with the bagged ensemble, and J+aB with XGBoost \footnote{Results for all models can be found on our GitHub.}.
These ML methods are chosen because they achieve the desired coverage and have the smallest average width across the test sets of the 17 datasets.
We emphasize that these choices require oracle knowledge of test-set performance, which would not be available to a practitioner;
PCS-UQ's prediction-check is designed to simulate this oracle in practice.
Therefore, although model checking is not explicitly considered in split conformal, Studentized conformal, and J+aB, we impose a ``global'' (i.e. averaged-across-all-datasets) oracle-selection on these methods in our comparison results (instead of comparing with conformal approaches for each of the candidate algorithms).

\paragraph{All methods achieve desired marginal coverage} Test-set coverage is reported for all methods and datasets in \cref{tab:coverage_regression}. All conformal regression methods and PCS achieve the target $90\%$ coverage for every dataset. 

\paragraph{PCS produces matching or smaller marginal intervals than oracle-selected conformal baselines} 

\cref{fig:regression_width_comparison} displays the average interval width across the 17 datasets, along with the distribution of per-dataset percentage reductions.
PCS produces intervals that are 10\% - 20\% shorter on average than those of globally oracle-selected split and Studentized conformal, and are 5\% shorter on average than oracle-selected J+aB conformal.

The comparison with J+aB merits closer examination, since it is the baseline most structurally similar to PCS-UQ: both use the bootstrap for stability, and both, as presented here, operate with top-performing predictive algorithms (with PCS-UQ's top choice being adaptive to each dataset).
The essential methodological difference is that J+aB uses a constant additive offset for calibration, while PCS-UQ uses multiplicative scaling. 
This difference is small in marginal terms (5\% shorter for PCS-UQ) but, as the next result shows, consequential for subgroup coverage.

\begin{figure}[htbp]
    \centering
    \includegraphics[width=0.48\textwidth]{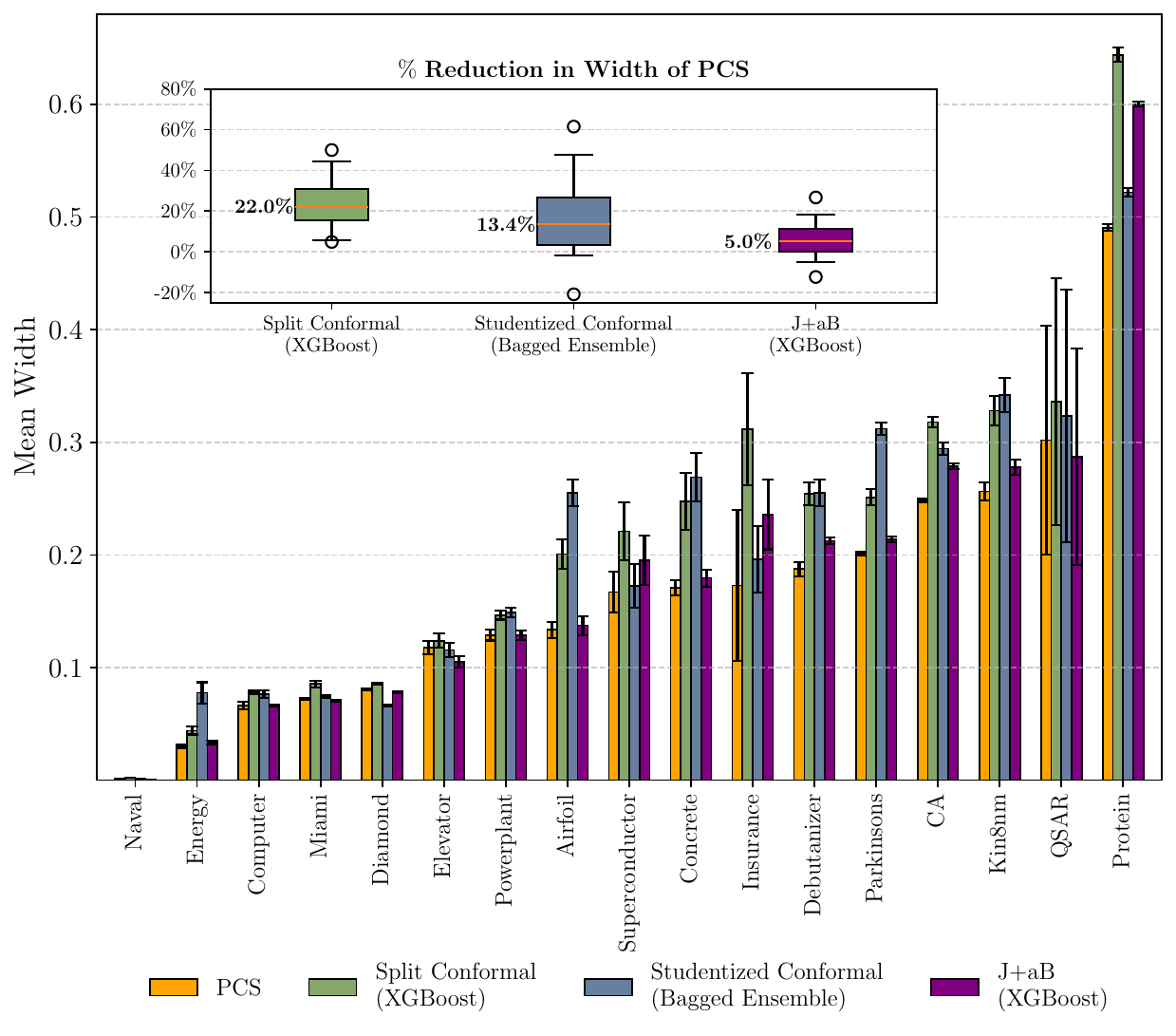}
    \caption{Comparison of PCS against three best-performing conformal methods: Split conformal (XGBoost), Studentized conformal (Bagged Ensemble), J+aB (XGBoost) across 17 datasets. We display the distribution of $\%$ improvement of PCS in the inset plot. PCS-UQ displays consistently better performance in interval width than conformal approaches in 15 out of 17 datasets (and comparable performance in the other two datasets). }
    \label{fig:regression_width_comparison}
\end{figure}

\paragraph{PCS adapts to subgroup structure} 
\label{subsec:subgroup_resuls}
Practitioners are often interested in whether intervals remain valid on heterogeneous subgroups, not only on average. 
For each dataset in \cref{tab:datasets}, we construct natural subgroups (details in \cref{supp:subgroups}) and evaluate coverage and width per subgroup.
Importantly, PCS-UQ has no knowledge of the subgroup definitions during training or calibration, so any subgroup adaptivity it exhibits is not the result of tuning.

\cref{fig:regression_subgroup_cov_width} shows the distribution of average subgroup coverage and width across datasets.
PCS-UQ consistently meets the 90\% target across subgroups while maintaining small average subgroup width. 
J+aB — marginally the strongest baseline — under-covers in some subgroups.
Studentized conformal, which models local residuals explicitly, maintains subgroup coverage but at the cost of wider marginal intervals than PCS-UQ. 
Split conformal under-covers on several subgroups. 
PCS-UQ is the only method tested that achieves both competitive marginal width and consistent subgroup coverage across this benchmark. The pattern holds at the individual-dataset level as well (\cref{supp:subgroups}).

\begin{figure}[htbp]
    \centering
    \includegraphics[width=0.48\textwidth]{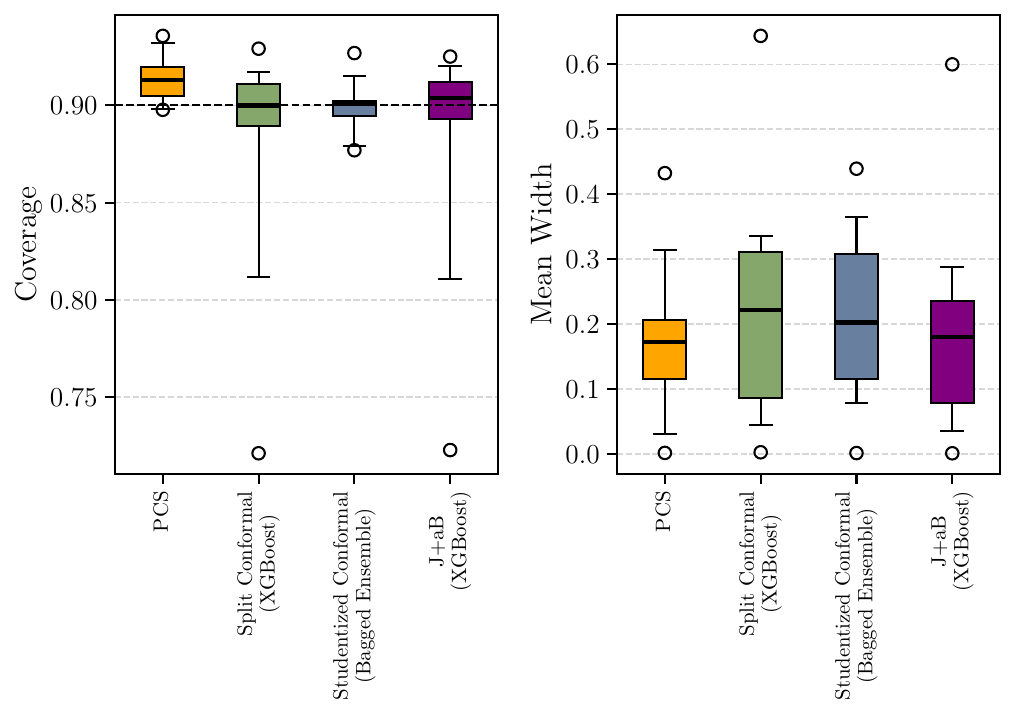}
    \caption{Distributions of average subgroup coverage and width for PCS and conformal regression approaches. For each dataset, we average the test coverage from each subgroup. PCS-UQ maintains subgroup average coverage while producing small average width.}
    \label{fig:regression_subgroup_cov_width}
\end{figure}
\section{PCS-UQ for Multi-Class Classification}
\label{sec:pcs_mc}
We detail the PCS-UQ procedure for generating prediction sets in the multi-class classification setting. 

\paragraph{Notation} We adopt much of the same notation as in \cref{sec:PCS_reg}, except that we assume responses belong to one of $C$ classes, i.e., $y \in \mathcal{Y} = \{1,\ldots, C\}$.
Additionally, let $\hat{f}^{(c)}(;)$ be the predicted probability that a sample belongs to class $c \in \mathcal{Y}$. 
Lastly, for class-probability estimates $\hat{y}^{(1)} \ldots \hat{y}^{(C)}$, let $\hat{y}^{\pi(1)} \ldots \hat{y}^{\pi(C)}$ denote the order statistics. 

\paragraph{Step 1: Data-Splitting and Prediction-Check} Repeat step 1 from \cref{sec:PCS_reg}.

\paragraph{Step 2: Bootstrapping}Repeat step 2 from \cref{sec:PCS_reg}.

\paragraph{Step 3: Generate Uncalibrated Predictions} First, for each $(\mathbf{X}_i, Y_i)$, compute the mean prediction across all bootstrapped models. That is, for class $c \in \mathcal{Y}$, let 
\begin{equation}
\label{eq:pcs_multiclass_pred}
    \hat{y}^{(c)}_i = \frac{1}{|T_i|k} \sum_{j \in [k] } \sum _{b \in T_i} \hat{f}^{c}_{j}(\mathbf{X}_i;\,\mathcal{D}^{(b)}),
\end{equation}
where recall that $T_i$ denotes bootstrap indices where $(\mathbf{X}_i,Y_i)$ is out-of-bag. 
This is similar to the ensemble method proposed in chapter 13 of Yu and Barter \cite{yu2024veridical}.

\paragraph{Step 4: Calibration} We follow the adaptive prediction set (APS) procedure introduced in \cite{romano2020classification}. Obtain APS score $S_i = \sum_{c=1}^{r} \hat{y}^{\pi(c)}_i$, where $\pi(r)  = Y_i$. For $\mathcal{S} = \{ S_i: i \in \mathcal{D}\}$, let $q$ denote the $1-\alpha$ quantile of $\mathcal{S}$.

\paragraph{Step 5: Generating PCS Prediction Sets for Test Point}  For a new test point $\mathbf{X}$, produce the prediction set
\begin{equation*}
    \begin{aligned}    
    \mathcal{S} = \{\pi(1), \pi(2), \dots, \pi(r)\},\,\,\, \text{where } r = \min \bigg\{t: \sum_{c=1}^{t} \hat{y}^{\pi(c)} \geq q \bigg\}
    \end{aligned}
\end{equation*}

\section{Multi-Class Classification Experiments}
\label{sec:class_exp}

\subsection{Experimental Set-up}
\label{subsec:class_exp_setup}

This section details the experimental set-up for our multi-class classification experiments.

\paragraph{Datasets} We use $6$ datasets commonly found in tabular benchmarks \cite{matthias2021openml}. 
These datasets reflect a range of sample sizes, dimensions, and number of classes. 
We use $80\%$ to train and fit various UQ methods, and $20\%$ as our test set.

\begin{table}[ht]
    \centering
    \begin{tabular}{clrrr}
    \toprule
    & Name &  Samples &  Features & Classes\\
    \midrule
        & Language \cite{collins2003collins} & 1000 &  19 & 30 \\
        &  Yeast \cite{horton1996probabilistic} &  1484  & 8 & 10 \\
        &  Isolet \cite{cole1991isolet} &  7797  & 613 & 26 \\
        & Cover Type \cite{blackard1998covertype} & 10000 & 13 & 100 \\
        & Chess \cite{alcalafdez2011keel} & 28056 & 34 & 18 \\
        & Dionis \cite{guyon2019analysis} & 30000 & 60 & 355 \\
     \bottomrule
     \end{tabular}
    \caption{Datasets used for multi-class classification experiments.}
    \label{tab:class_datasets}
\end{table}
\paragraph{Baseline Methods} We compare PCS against three popular conformal multi-class classification methods: Adaptive Prediction Sets (APS) \cite{romano2020classification}, Regularized Adaptive Prediction Sets (RAPS) \cite{angelopoulos2022uncertaintysetsimageclassifiers}, and Top$K$ \footnote{We do not compare to J+aB as we cannot find an implementation of the method for classification.}.
We use the implementation of all conformal methods from the software package \texttt{MAPIE} \cite{Cordier2023Flexible}.
We generate prediction sets for all methods with the following ML models: $\ell_2$-regularized logistic regression, Random Forests \cite{breiman2001random}, AdaBoost \cite{freund1997decision}, XGBoost \cite{chen2016xgboost}, and a one-hidden-layer multi-layer perceptron (MLP). 
We choose regularization parameters in $\ell_2$-regularized logistic regression via $3$-fold cross-validation.
For other ML models, we use the default hyperparameters from \texttt{scikit-learn} \cite{scikit-learn}. 

\paragraph{PCS Hyperparameters} We use all models listed above as candidate models, and choose $k=1$ as we did in the regression experiments; see \cref{sec:PCS_reg} for a description. 
We generate intervals using $B  = 1000$ bootstraps.

\paragraph{Metrics} We measure coverage and average size of prediction sets on the test set.
We aim for $90\%$ coverage, i.e., we set $\alpha = 0.1$.
The size of prediction sets is normalized by the number of classes, $C$. 
Results are averaged across 10 train-test splits. 

\subsection{Results}  
\label{subsec:class_results}

This section details results for experiments described in the previous section. 
For APS, RAPS, and Top$K$, we report performance using Random Forests as the estimator since it achieves coverage and has the smallest prediction set size on average across our $6$ datasets. 
We emphasize that we choose the best-performing estimators for conformal methods --- information that is unavailable in practice.

\paragraph{All Methods Achieve Desired Coverage} Test-set coverage is reported for all methods and datasets in \cref{tab:coverage_classification}. All methods achieve the desired coverage.

\paragraph{PCS Produces Smaller Sets than Conformal Approaches} \cref{fig:class_width_comparisons} displays average prediction set size for all methods. PCS produces smaller average prediction set size than all other methods on five out of six datasets. The table below summarizes the median reduction in set size by PCS over all conformal methods.

\begin{figure}[h]
    \centering
    \includegraphics[width=0.45\textwidth]{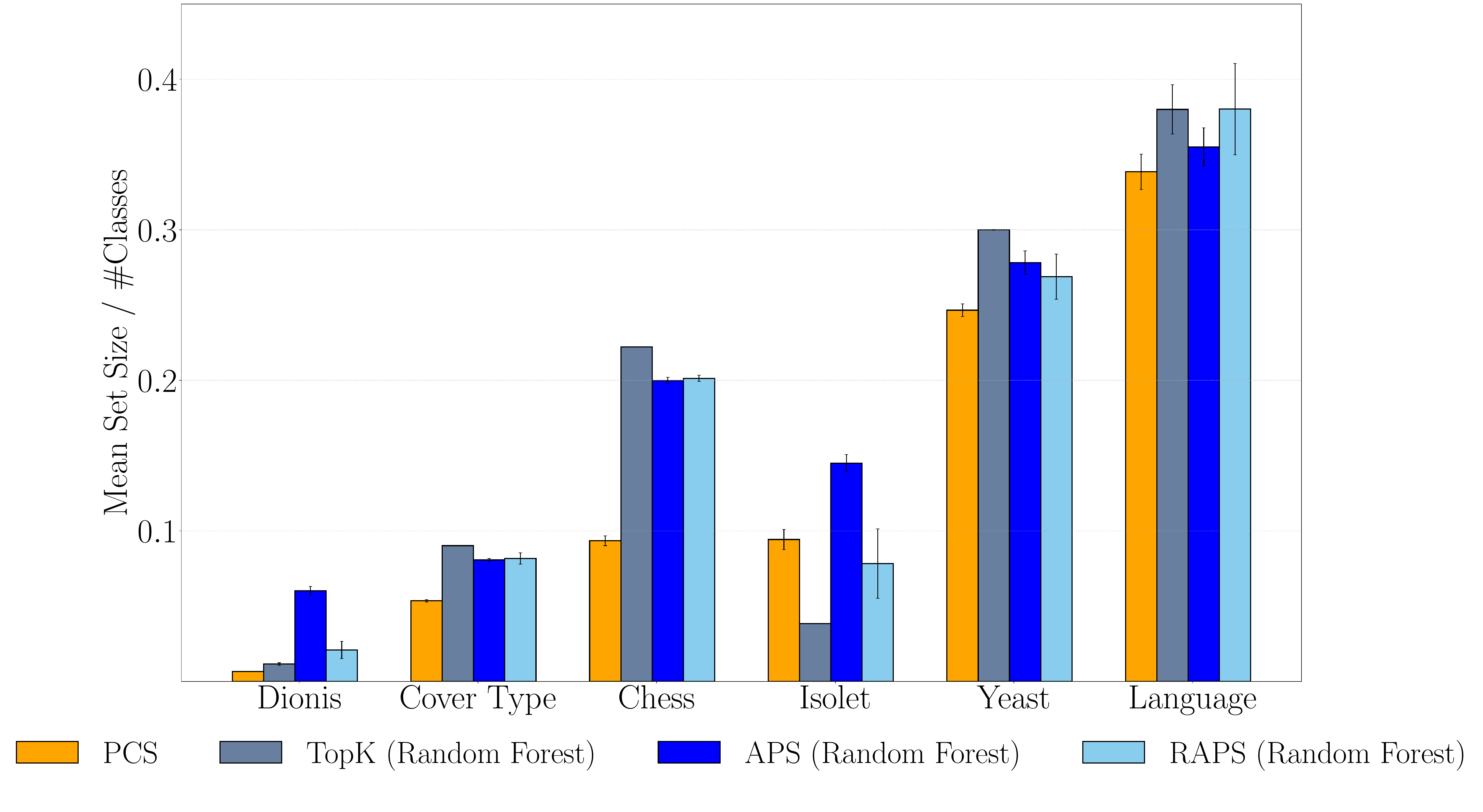}
    \caption{Comparison of average prediction set size of PCS against best-performing conformal methods. PCS significantly reduces set size across 5 out of 6 datasets. }
    \label{fig:class_width_comparisons}
\end{figure}

\begin{table}[h]
\centering
\begin{tabular}{rrr}
\toprule
 Top$K$ (RF) & APS (RF) & RAPS (RF)\\
\midrule
29.2\% & 34.3\% & 22.7\% \\
\bottomrule
\end{tabular}
\caption{Median \% reduction in set size by PCS over best-performing conformal approaches across 6 multi-class classification datasets.}
\end{table}

 \section{PCS Uncertainty Quantification for Deep-Learning}
\label{sec:pcs_dl}
While PCS significantly reduces set size in our experiments, training multiple bootstrapped models can be prohibitively expensive for large deep-learning models. 
In this section, we discuss computationally efficient methods for generating prediction sets for deep-learning models via PCS and experimental results on large-scale deep-learning datasets. 

\subsection{Approximate PCS-UQ}
\label{subsec:approx_schemes}
Instead of training DL models across $B$ different bootstrapped datasets, we proceed as follows. 
First, we perform a simple train-calibration data split and train a \emph{single} DL model on the training set $\mathcal{D}_{\text{train}}$. 
Throughout this description, we assume that the DL model achieves sufficient predictive accuracy outside the training set.
If not, we recommend trying a different DL architecture or training algorithm. 
This emphasizes that establishing strong predictability is key for trustworthy UQ and, ultimately, for trustworthy ML and AI deployment.
Next, we create $B$ perturbed DL models as follows: (1) \textit{Weighted Monte-Carlo Dropout.}
We create $B$ perturbed models by randomly dropping out nodes in a DL model \cite{gal2016dropout}. 
The probability of dropout is set to be proportional to the activation. 
(2) \textit{Additive Noise Perturbation.} We create $B$ perturbed models by adding mean-zero Gaussian noise to the weights \cite{zhou2024knowgraph, gan2026neural}. 
The variance of the added noise is set to be the initialization variance.  

\subsection{Experimental Results}
We perform experiments comparing the original PCS-UQ for multi-class classification, our approximation methods, and the conformal inference methods on three computer vision benchmarks.

\paragraph{Datasets} We use the three following standard computer vision benchmarks. Descriptions of the datasets are in \cref{supp:dl}.
Summary statistics of the datasets are as follows.

\begin{table}[h]
    \centering
    \begin{tabular}{lrr}
    \toprule
    Name &  Samples & Classes\\
    \midrule
        CIFAR-100 \cite{krizhevsky2009learning} & 60000 & 100 \\
        Caltech Birds \cite{welinder2010caltech} &  11788 & 200 \\
        ImageNet-Small \cite{le2015tiny} &  100000 & 200 \\
     \bottomrule
     \end{tabular}
    \caption{Datasets used for deep-learning classification experiments. }
    \label{tab:dl_datasets}
\end{table}

\paragraph{Model Details} For all datasets, we use a ResNet 18 \cite{he2016deep}.

\paragraph{UQ Methods} We compare PCS-UQ for multi-class classification 
described in \cref{sec:pcs_mc}, and the approximation methods described above. 
For the original multi-class PCS-UQ, we use $B = 100$ bootstraps.
Since PCS-UQ does not require a separate validation set due to the use of OOB samples, we combine the training and validation sets. 
For the approximation methods described above, we create $B = 1000$ models.

\paragraph{Metrics} We aim for $90\%$ coverage, and measure average prediction set size on the test set. 
Further, we measure the time taken (rounded to the nearest minute) to produce prediction sets for each UQ method. 
Results are averaged across 10 train-test splits.

\paragraph{Results} The results for each dataset are presented as follows. 
All UQ methods achieve the desired coverage (See \cref{supp:dl}). 
Original PCS-UQ improves upon conformal methods by producing prediction sets that are $26\%$ smaller on average. 
Both PCS approximation methods improve upon conformal methods by approximately $20\%$, but do not match the performance of the original PCS method.
However, the approximation schemes are approximately $30 \text{ to } 100\times$ faster than the original PCS method. 
As such, the approximation methods strike a balance between computational efficiency and improving the size of prediction sets.


\begin{table}[t]
\centering
\small
\setlength{\tabcolsep}{5pt}
\renewcommand{\arraystretch}{1.2}
\begin{tabular}{llccccccc}
\toprule
& & \multicolumn{2}{c}{\textbf{CIFAR-100}} & \multicolumn{2}{c}{\textbf{Caltech Birds}} & \multicolumn{2}{c}{\textbf{ImageNet-Small}} \\
& \textbf{Method} & Av. Size & Time & Av. Size & Time & Av. Size & Time \\
\midrule
& \textcolor{blue}{APS} & 6.8 & 2 & 16.6 & 2 & 14.4 & 3 \\
& \textcolor{cyan}{RAPS} & 6.5 & 2 & 12.8 & 2 & 11.2 & 3 \\
& TopK & 8.5 & 2 & 17.6 & 2 & 13 & 3 \\
\midrule
\multirow{3}{*}{\rotatebox[origin=c]{90}{\textcolor{orange}{PCS}}}
& Original & \textbf{3.7} & 350 & \textbf{8.3} & 100 & \textbf{8.8} & 500 \\
& Dropout & 4.4 & 4 & 9.8 & 3 & 9.8 & 5 \\
& Noise & 4.2 & 3 & 9.4 & 3 & 9.6 & 5 \\
\bottomrule
\end{tabular}
\caption{Average prediction set size and runtime (minutes) across multiple computer vision benchmarks. Set sizes are normalized by the number of classes. PCS-UQ outperforms conformal approaches in terms of prediction set size. Both proposed approximation schemes strike a balance between computational efficiency and size of prediction sets.}
\label{tab:pcs_dl}
\end{table}

\section{Connection to conformal inference}
\label{sec:theory}
This section formalizes connections between PCS-UQ and conformal prediction, as alluded to in \cref{sec:intro}.
We start by discussing how the multiplicative calibration in PCS-UQ can be regarded as a conformal score function. 
Then, we utilize this connection to theoretically establish that a modified PCS-UQ algorithm achieves the desired coverage with exchangeable data. 

\paragraph{Multiplicative Calibration as Conformal Score Function}
Conformal prediction relies on specifying a score function that measures the quality of the prediction. For example, residuals are typically used in regression as a valid conformal score. 
We show that the multiplicative calibration step in PCS-UQ, i.e., $\gamma$, can be regarded as a novel conformal score function.  
In our setting, a larger $\gamma$ indicates poorer prediction.

\paragraph{Modified PCS-UQ achieves desired coverage} 
The algorithm proposed in \cref{sec:PCS_reg} uses the validation data for screening prediction algorithms and calibration. 
Doing so makes it difficult to establish theoretical guarantees that the produced interval is statistically valid. 
The modified PCS-UQ procedure overcomes this issue by randomly splitting the data into training, validation, and calibration sets.
The training and validation data are used for prediction-check and fitting the bootstrapped models, while the calibration set is used \emph{solely} to learn the scaling factor $\gamma$.
With this modified algorithm and the connection to conformal score functions detailed above, we utilize previous results showing that any prediction interval formed using a valid score function achieves the desired coverage. 
The formal result is presented in \cref{supp:theory}.
This connection suggests that other PCS components, such as prediction-check, could be incorporated into conformal pipelines more broadly while retaining validity guarantees.

\section{Ablation Experiments}
\label{sec:ablation}
As discussed in \cref{sec:PCS_reg}, generating PCS prediction intervals consists of a few key steps: prediction-checking, bootstrapping, and multiplicative calibration. 
We perform a number of ablation experiments to demonstrate the utility of each of these steps as follows.

\paragraph{Effect of prediction-check} We vary the number of screened models (i.e., $k$) and measure average width and $R^2$ on the test set for $4$ datasets.
Results are displayed in \cref{fig:ablation-model}, which shows that including poor-performing models leads to larger intervals -- highlighting the importance of screening models via their prediction performance.
For all 4 datasets, using the top 1 algorithm gives the best performance.
However, for the Diamond dataset, the performances obtained using the top one, top two, and top three algorithms are all similar.
This indicates that it could be beneficial to dynamically choose the number of algorithms to capture more diverse structures of data (e.g. heterogeneity).
One way to achieve this is via the Model Confidence Set, whose goal is to select the subset of all nearly optimal algorithms \cite{lei2025moderntheorycrossvalidationlens}.
We leave this modification for future work.

\begin{figure}[htbp]
    \centering
    \includegraphics[width=0.48\textwidth]{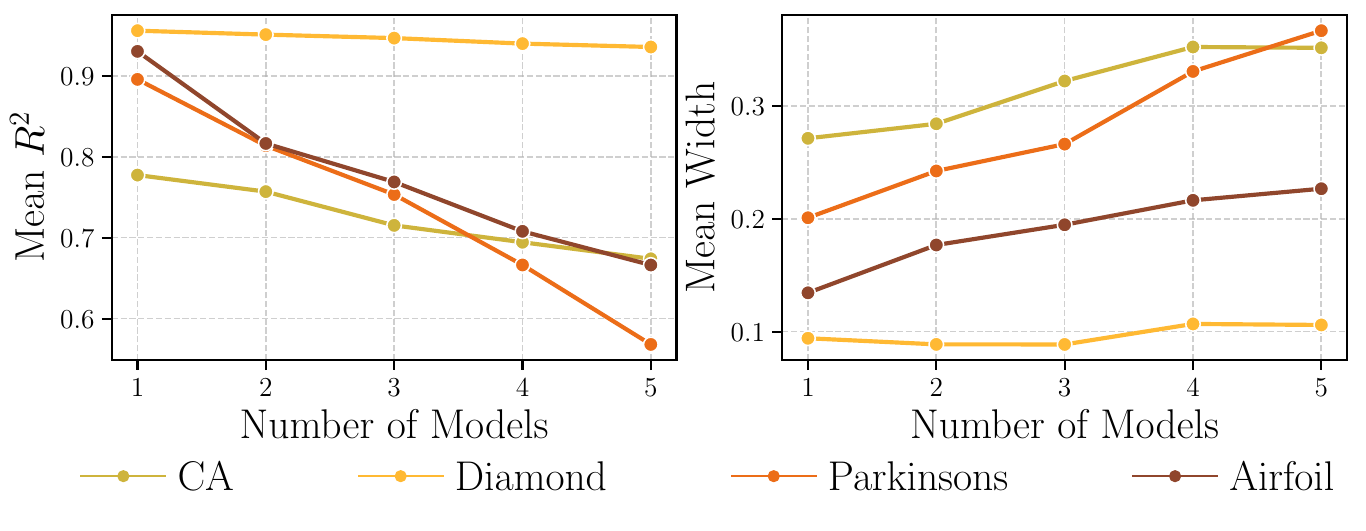}
    \caption{Performance of PCS with varying number of selected models over 4 datasets. The left panel displays the average $R^2$ of selected models; the right panel displays the average interval width. As the number of selected models increases, the $R^2$ decreases while the interval width increases.}
    \label{fig:ablation-model}
\end{figure}

\paragraph{Varying the number of bootstraps} 

A key step in the PCS procedure is creating a pseudo-universe of datasets via the bootstrap.
In \cref{fig:ablation-nboot}, we display the average interval size and coverage as we vary the number of bootstraps.
Performance stabilizes after 100 bootstraps.
Bootstrapping allows one to simulate and capture uncertainty during the data collection process. 

\begin{figure}[htbp]
    \centering
    \includegraphics[width=0.48\textwidth]{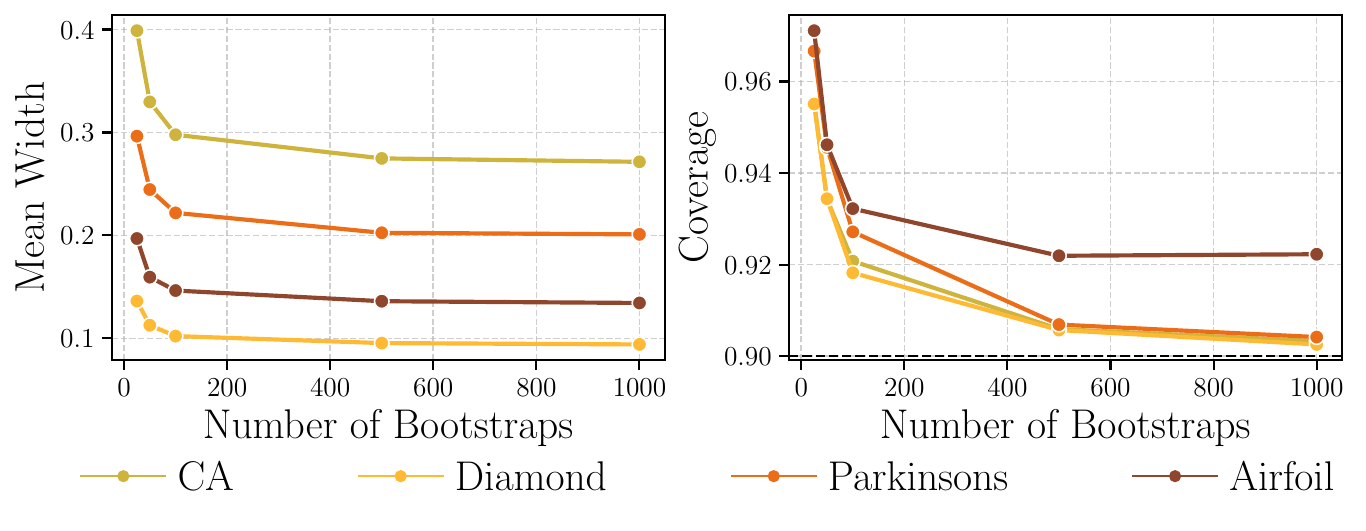}
    \caption{Performance of PCS-UQ with varying number of bootstraps over 4 datasets. The left panel displays the average interval width; the right panel displays the coverage. Both metrics stabilize after 100 bootstraps.}
    \label{fig:ablation-nboot}
\end{figure}

\paragraph{Multiplicative calibration} To investigate the effectiveness of multiplicative calibration for subgroup coverage, we replace multiplicative calibration with additive calibration. 
That is, we enlarge intervals by adding a fixed constant to both ends, instead of scaling the interval widths multiplicatively.
We examine subgroup coverage for the Miami housing dataset \cite{bourassa2021miami} as in \cref{sec:reg_exp}. 
\cref{fig:ablation_calibration} shows that additive calibration is unable to achieve target coverage for houses larger than $2850$ square feet. 
Similar results hold for other datasets. 
Additive calibration is unable to enlarge intervals sufficiently for samples with high uncertainty, while multiplicative scaling does so effectively. 

\begin{figure}[htbp]
    \centering
    \includegraphics[width=0.48\textwidth]{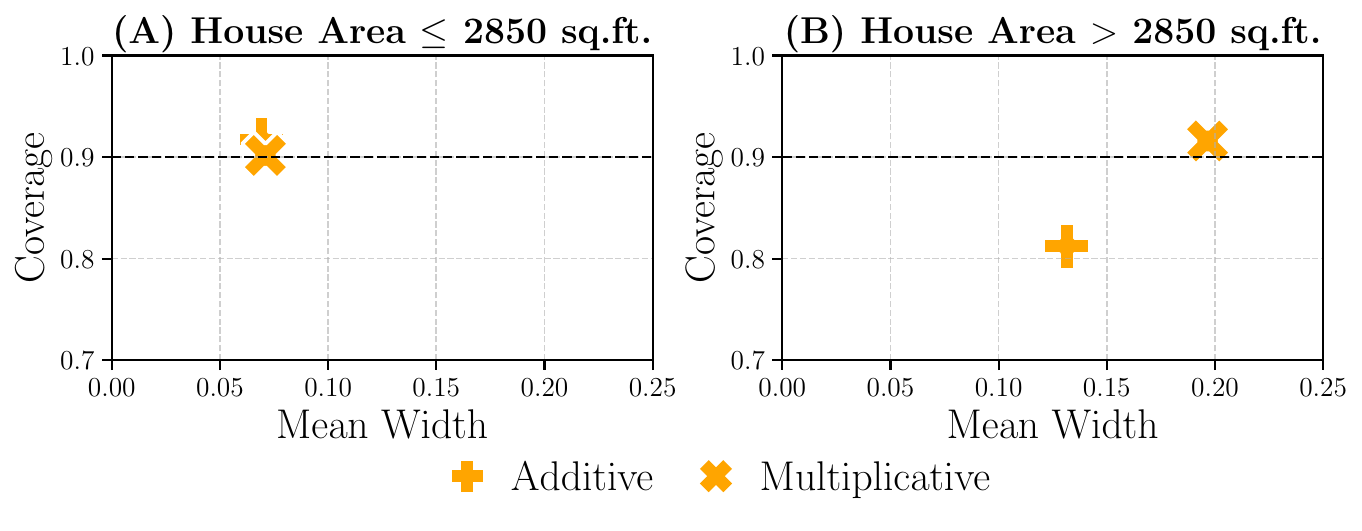}
    \caption{Subgroup coverage of additive and multiplicative calibration on the Miami housing dataset \cite{bourassa2021miami}. Additive calibration is unable to achieve target coverage for large houses, while multiplicative calibration adjusts length to do so effectively. }
    \label{fig:ablation_calibration}
\end{figure}

\section{Discussion}
\label{sec:disc}
Our approach builds upon key PCS principles to develop prediction intervals.
Extensive empirical comparisons of PCS-UQ to conformal methods show we reduce the size of prediction sets by over $20\%$ across a variety of settings. 
Our paper also establishes theoretical connections to conformal inference that might be of independent interest. 
While our paper takes a step towards establishing PCS-driven UQ, there are many extensions and improvements to explore for future work. 
We detail some of these as follows.

\paragraph{Uncertainty from Data Cleaning \& Judgment Calls} This paper only focuses on uncertainty due to inter-sample variability. 
As discussed in \cref{sec:intro}, data cleaning choices and other judgment calls due to inter-researcher variability can lead to drastically different conclusions. 
An interesting direction for future work is to find approaches to assess uncertainty from every part of the DSLC to create a more stable UQ method.  
Some work in this direction has been done in a follow-up paper by Yu and her other collaborators, which devises a method called CLEAR that combines PCS-UQ with the Conformal Quantile Regression (CQR) method to capture both epistemic and aleatoric uncertainties \cite{azizi2026clear}.

\paragraph{Extension to Binary Classification} Our approach for classification produces only prediction sets, which are often unsuitable for binary classification. 
In the binary setting, producing intervals that contain the true $\mathbb{P}(Y = 1 | \mathbf{X})$ is often more relevant to practitioners. 
As a simple example, producing intervals that state there is a $40\%-60\%$ chance of rain is more instructive than producing a prediction set that consists of both rain and no rain. 
Constructing intervals for the underlying class probability is difficult because we do not observe empirical class probabilities, but only binary labels. 
Observations of class labels also make evaluation of probability intervals challenging. 

\paragraph{Extension to LLMs and Generative Models} An exciting future direction is using PCS to assess the uncertainty of LLMs and other generative models. 
Doing so requires defining appropriate notions of prediction sets and coverage. 
We believe that robust UQ for LLMs and generative models has the potential to reduce hallucinations and improve factuality \cite{cherian2024largelanguagemodelvalidity}, thus enhancing trustworthiness of AI systems.

\section{Acknowledgements}
\label{sec:ack}
We thank Rina Foygel Barber, Giles Hooker, Jing Lei, Anthony Ozerov, Aaditya Ramdas, Jake A. Soloff, and Ryan Tibshirani for insightful comments and useful suggestions. 
We also gratefully acknowledge partial support from NSF grant DMS-2413265, NSF grant DMS 2209975, NSF grant DMS-2515767, NSF grant 2023505 on Collaborative Research: Foundations of Data Science Institute (FODSI), the NSF and the Simons Foundation for the Collaboration on the Theoretical Foundations of Deep Learning through awards DMS-2031883 and 814639, NSF grant MC2378 to the Institute for Artificial CyberThreat Intelligence and OperatioN (ACTION), and NIH grant R01GM152718.

\bibliography{ref}
\appendix
\setcounter{figure}{0}
\setcounter{table}{0}

\renewcommand{\thetable}{S\arabic{table}}
\renewcommand{\thefigure}{S\arabic{figure}}
\renewcommand{\thesection}{S\arabic{section}}

\newpage
\onecolumn

\section{Overview of Uncertainty Quantification Methods for Regression}
\label{supp:overview_different_methods_reg}

\subsection{Split Conformal Regression}
\label{supp:split_conformal_overview}

We describe the Split Conformal procedure from \cite{lei2018distribution}.
We use the same notation as established in Section \ref{sec:PCS_reg}.

\paragraph{Step 1: Data-Splitting and Model Training} Randomly split $\mathcal{D}$ into a training set $\mathcal{D}_{\text{tr}}$ and validation set $\mathcal{D}_{\text{val}}$. Fit algorithm $f$ on training set to obtain fitted model $\hat{f}(\cdot; \mathcal{D}_{\text{tr}})$. 

\paragraph{Step 2: Calibration} For each $(\mathbf{X}_i, Y_i) \in \mathcal{D}_{\text{val}}$, make prediction using $\hat{f}(\cdot; \mathcal{D}_{\text{tr}})$ and obtain conformal score $S_i = \big|Y_i - \hat{f}(\mathbf{X}_i; \mathcal{D}_{\text{tr}})\big|$. 
Let $q$ be the $1 - \alpha$ quantile of the set $\{S_i: i \in [|\mathcal{D}_{\text{val}}|]\}$.

\paragraph{Step 3: Generate Split Conformal Prediction Interval} For a new test point $\mathbf{X}$, produce the prediction interval
\begin{equation*}
    \mathcal{I} = \big[\hat{f}(\mathbf{X}; \mathcal{D}_{\text{tr}}) - q, \hat{f}(\mathbf{X}; \mathcal{D}_{\text{tr}}) + q\big]
\end{equation*}

\subsection{Studentized Conformal Regression}
\label{supp:studentized_overview}

We describe the Studentized Conformal procedure from \cite{lei2018distribution}.
We use the same notation as established in Section \ref{sec:PCS_reg}.

\paragraph{Step 1: Data-Splitting and Model Training} Randomly split $\mathcal{D}$ into a training set $\mathcal{D}_{\text{tr}}$ and validation set $\mathcal{D}_{\text{val}}$. Fit algorithm $f$ on training set to obtain fitted model $\hat{f}(\cdot; \mathcal{D}_{\text{tr}})$. 
Then let $\mathcal{D}_{\text{tr}}^{\text{res}} = \{(\mathbf{X}_i, |Y_i - \hat{f}(\mathbf{X}_i; \mathcal{D}_{\text{tr}})|): (\mathbf{X}_i, Y_i) \in \mathcal{D}_{\text{tr}}\}$ be the training set with residuals as the response. 
Fit algorithm $\sigma$ on $\mathcal{D}_{\text{tr}}^{\text{res}}$ to obtain fitted model $\hat{\sigma}(\cdot; \mathcal{D}_{\text{tr}}^{\text{res}})$.

\paragraph{Step 2: Calibration} For each $(\mathbf{X}_i, Y_i) \in \mathcal{D}_{\text{val}}$, make predictions using $\hat{f}(\cdot; \mathcal{D}_{\text{tr}})$, $\hat{\sigma}(\cdot; \mathcal{D}_{\text{tr}}^{\text{res}})$; obtain conformal score
\begin{equation*}
    S_i = \frac{\big|Y_i - \hat{f}(\mathbf{X}_i; \mathcal{D}_{\text{tr}})\big|}{\hat{\sigma}(\mathbf{X}_i; \mathcal{D}_{\text{tr}}^{\text{res}})}
\end{equation*}
Let $q$ be the $1 - \alpha$ quantile of the set $\{S_i: i \in [|\mathcal{D}_{\text{val}}|]\}$.

\paragraph{Step 3: Generate Studentized Conformal Prediction Interval} For a new test point $\mathbf{X}$, produce the prediction interval
\begin{equation*}
    \mathcal{I} = \big[\hat{f}(\mathbf{X}; \mathcal{D}_{\text{tr}}) - q \times \hat{\sigma}(\mathbf{X}; \mathcal{D}_{\text{tr}}^{\text{res}}), \hat{f}(\mathbf{X}; \mathcal{D}_{\text{tr}}) + q \times \hat{\sigma}(\mathbf{X}; \mathcal{D}_{\text{tr}}^{\text{res}})\big]
\end{equation*}

\subsection{Jackknife+-after-Bootstrap}
\label{supp:jackknife_bootstrap}

We describe the Jackknife+-after-bootstrap procedure from \cite{kim2020predictive}.
We use the same notation as established in Section \ref{sec:PCS_reg}.

\paragraph{Step 1: Bootstrap} Bootstrap $\mathcal{D}$, the training data, $B$ times to obtain bootstrapped datasets $\mathcal{D}^{(1)}, \dots, \mathcal{D}^{(B)}$. 
Fit algorithm $f$ to each bootstrapped dataset to obtain bootstrapped models $\hat{f}(\cdot; \mathcal{D}^{(1)}), \dots, \hat{f}(\cdot; \mathcal{D}^{(B)})$. 
For each training point $(\mathbf{X}_i, Y_i) \in \mathcal{D}$, let $T_i \subseteq [B]$ denote the set of bootstrap indices such that $(\mathbf{X}_i, Y_i) \notin \mathcal{D}^{(b)}$ for all $b \in T_i$. This is the set of bootstrapped models where $(\mathbf{X}_i, Y_i)$ is out-of-bag.

\paragraph{Step 2: Compute Residuals} For any index $i$ of the training dataset, define
\[
\hat{f}_{-i}(x) = \frac{1}{|T_i|} \sum_{b \in T_i} \hat{f}(x; \mathcal{D}^{(b)}).
\]
This is the average prediction of $x$ using all models where $(\mathbf{X}_i, Y_i)$ is out-of-bag. 
Then, for $(\mathbf{X}_i, Y_i) \in \mathcal{D}$, define the residuals
\[
R_i = |Y_i - \hat{f}_{-i}(\mathbf{X}_i)|.
\]

\paragraph{Step 3: Generate Jackknife+-after-Bootstrap Prediction Interval} For a new test point $\mathbf{X}$, define the sets $-\mathcal{P}_{\text{lower}} = \{-(\hat{f}_{-i}(\mathbf{X}) - R_i): i=1,2,\dots, n_{\text{train}}\}$, and $\mathcal{P}_{\text{upper}} = \{\hat{f}_{-i}(\mathbf{X}) + R_i: i=1,2,\dots, n_{\text{train}}\}$.
Note that for each $i$, we are not using all bootstrapped models, but only ones where training point $i$ is out-of-bag.
Then with desired coverage $1-\alpha$, the prediction interval is
\[
\left[-q_{1 - \alpha}(-\mathcal{P}_{\text{lower}}), q_{1-\alpha}(\mathcal{P}_{\text{upper}})\right].
\]

\subsection{PCS procedure from Chapter 13 of Yu and Barter (2024)}
\label{supp:PCS_ch13}

We describe the PCS procedure from Chapter 13 of \cite{yu2024veridical} for generating prediction intervals in the regression setting.  
Henceforth, we refer to this procedure as PCS (Ch 13). 
We use the same notation as established in Section \ref{sec:PCS_reg}.

\paragraph{Step 1: Data-Splitting and Prediction-Check} Randomly split $\mathcal{D}$ into a training set $\mathcal{D}_{\text{tr}}$ and validation set $\mathcal{D}_{\text{val}}$ . Train each algorithm on the training set to obtain fitted models $\hat{f}_1(\cdot; \mathcal{D}_{\text{tr}}), \ldots \hat{f}_M(\cdot; \mathcal{D}_{\text{tr}})$. Choose the top-$k$ performing algorithms according to loss $l$.  Without loss of generality, let $f_1 \ldots f_k$ denote the top-$k$ performing algorithms. 

\paragraph{Step 2: Bootstrapping} Bootstrap the \emph{training} set $B$ times to obtain bootstrapped samples $\mathcal{D}_{\text{tr}}^{(1)} \ldots \mathcal{D}_{\text{tr}}^{(B)}$. Fit all algorithms chosen in the previous step on every bootstrapped dataset $\mathcal{D}_{\text{tr}}^{(b)}$ to obtain bootstrapped models $\{\hat{f}_j(; \mathcal{D}_{\text{tr}}^{(b)}), j \in [k], b \in [B]\}$. 
For each $(\mathbf{X}_i, Y_i) \in \mathcal{D}_{\text{val}}$, we form a prediction set $\mathcal{P}_i = \{\hat{f}_j(\mathbf{X}_i; \mathcal{D}_{\text{tr}}^{(b)}); j \in [k], b \in [B]\}$.

\paragraph{Step 3: Calibration} First, for each $(\mathbf{X}_i, Y_i) \in \mathcal{D}_{\text{val}}$, we form an uncalibrated interval $[q_{\alpha/2}(\mathcal{P}_i), q_{1 - \alpha/2}(\mathcal{P}_i)]$, where $q_{\beta}(S)$ is the $\beta$ quantile for a set $S$. 
For a multiplicative scaling factor $\gamma$, generate a scaled interval
\begin{equation*}
    \mathcal{I}_i = \Big[q_{0.5} ({\cal P}_i) - \gamma \times 
    \big(q_{0.5} ({\cal P}_i) - q_{\alpha/2} ({\cal P}_i)\big), \quad q_{0.5} ({\cal P}_i) + \gamma \times 
    \big(q_{1-\alpha/2} ({\cal P}_i) - q_{0.5} ({\cal P}_i)\big) \Big]
\end{equation*}
We choose the scaling factor $\gamma$ such that we achieve $1-\alpha$ coverage on the data $\mathcal{D}_{\text{val}}$.

\paragraph{Step 4: Generating PCS Prediction Interval} For a new test point $\mathbf{X}$ let $\mathcal{P} = \{\hat{f}_j(\mathbf{X}; \mathcal{D}_{\text{tr}}^{(b)}); j \in [k], b \in B\}$. Then, we produce prediction interval 
\begin{equation*}
    \mathcal{I} = \Big[q_{0.5} ({\cal P}) - \gamma \times 
    \big(q_{0.5} ({\cal P}) - q_{\alpha/2} ({\cal P})\big), \quad q_{0.5} ({\cal P}) + \gamma \times 
    \big(q_{1-\alpha/2} ({\cal P}) - q_{0.5} ({\cal P})\big) \Big]
\end{equation*}

\section{Additional Regression Results}
\label{supp:reg_results}

In this section, we provide additional results for our regression experiments. 

\subsection{Coverage}
\label{supp:coverage}
We report coverage for the best-performing methods (as measured by average width) across our 17 real-world datasets.  
All methods achieve desired coverage. 

\begin{table}[H]
    \centering
    \begin{tabular}{lrrrr}
\toprule
Method & PCS & Split Conformal (XGBoost) & Studentized Conformal (Bagged Ensemble) & J+aB (XGBoost) \\
\midrule
Naval          & 0.905 & 0.906 & 0.902 & 0.905 \\
Energy         & 0.920 & 0.901 & 0.912 & 0.919 \\
Computer       & 0.902 & 0.901 & 0.902 & 0.900 \\
Miami          & 0.903 & 0.900 & 0.902 & 0.902 \\
Diamond        & 0.902 & 0.900 & 0.898 & 0.901 \\
Insurance      & 0.910 & 0.901 & 0.908 & 0.909 \\
Elevator       & 0.904 & 0.898 & 0.898 & 0.898 \\
Airfoil        & 0.917 & 0.914 & 0.910 & 0.906 \\
Powerplant     & 0.902 & 0.898 & 0.900 & 0.900 \\
Concrete       & 0.918 & 0.906 & 0.902 & 0.902 \\
Debutanizer    & 0.906 & 0.892 & 0.894 & 0.906 \\
Superconductor & 0.903 & 0.900 & 0.900 & 0.903 \\
Parkinsons     & 0.905 & 0.900 & 0.897 & 0.898 \\
CA             & 0.904 & 0.900 & 0.898 & 0.901 \\
Kin8nm         & 0.905 & 0.900 & 0.901 & 0.903 \\
QSAR           & 0.903 & 0.900 & 0.901 & 0.898 \\
Protein        & 0.902 & 0.899 & 0.899 & 0.899 \\
\bottomrule
\end{tabular}

    \caption{Coverage for PCS, and best-performing conformal methods across our $17$ real-world datasets. All methods achieve desired coverage. }
    \label{tab:coverage_regression}
\end{table}

\subsection{Additional Subgroup Results}
\label{supp:subgroups}

\paragraph{Subgroup Construction Procedure} We manually construct subgroups for each dataset.
First, we fit a Random Forest on the dataset and identify the feature with the highest importance. 
If the most important feature is binary or categorical, we form the subgroups by partitioning the data by category. 
If the feature is numerical, we inspect natural breaks in its distribution and partition the data accordingly (e.g. splitting between modes if the distribution is multi-modal). 
For a feature whose distribution does not have natural breaks, we partition the data into quartiles based on that feature.

Next, we provide subgroup results for additional datasets.

\paragraph{Miami Housing \cite{bourassa2021miami}} The task is to predict selling prices of houses in Miami.
We form subgroups based on the square footage of the house.
Split conformal (XGBoost), Studentized conformal (Bagged Ensemble), and J+aB (XGBoost) do not achieve the desired coverage in the large-area subgroup.
PCS-UQ adapts the width of its intervals to achieve coverage in both subgroups.

\begin{figure}[htbp]
    \centering
    \includegraphics[width=0.8\textwidth]{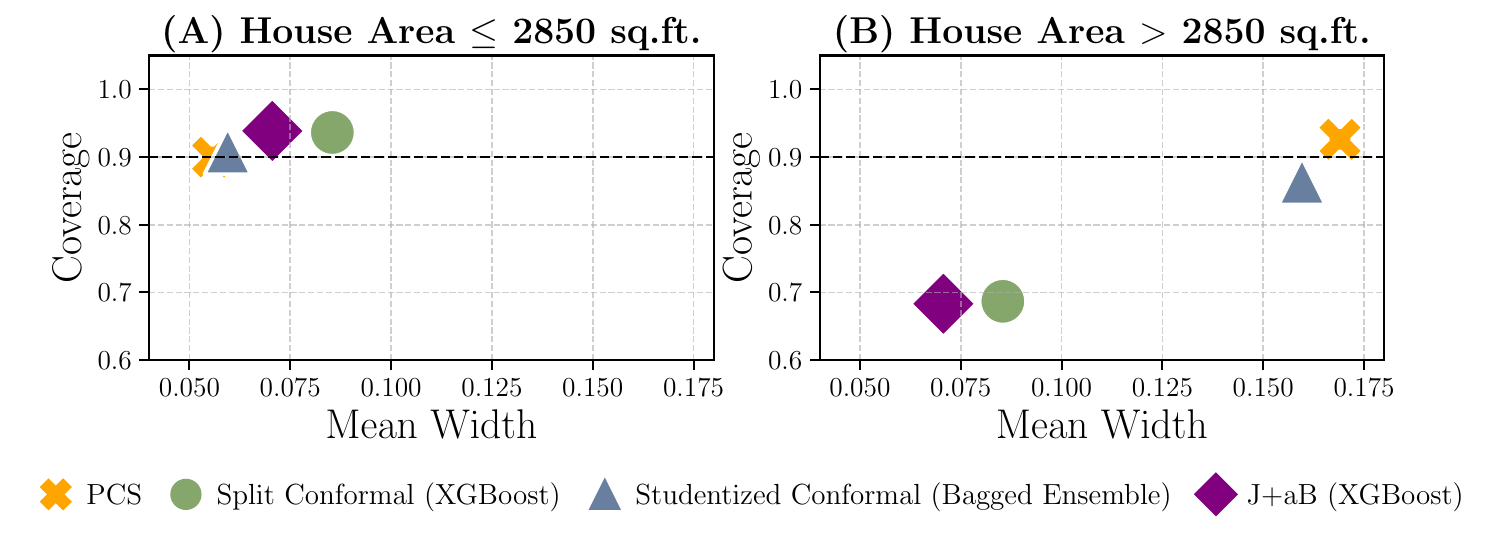}
    \caption{Coverage and width for PCS, and conformal regression approaches on subgroups in the Miami Housing dataset \cite{bourassa2021miami}. Panels (A) and (B) demonstrate performance on subgroups formed by square footage of the house. PCS adapts width of intervals to maintain coverage across subgroups. Other conformal methods either do not achieve subgroup coverage  or have larger width.}
    \label{fig:regression_subgroup_miami}
\end{figure}

\paragraph{Insurance \cite{ali2020pycaret}} The task is to predict insurance charges for customers.
We form subgroups by partitioning the data into smokers and non-smokers. 
PCS adapts its width to maintain coverage across subgroups. 
Split Conformal (XGBoost) and J+aB (XGBoost) achieve coverage but produce larger intervals. 
Studentized Conformal (Bagged Ensemble) fails to achieve coverage in the Smoker subgroup and produces larger intervals.
\begin{figure}[H]
    \centering
    \includegraphics[width=0.8\textwidth]{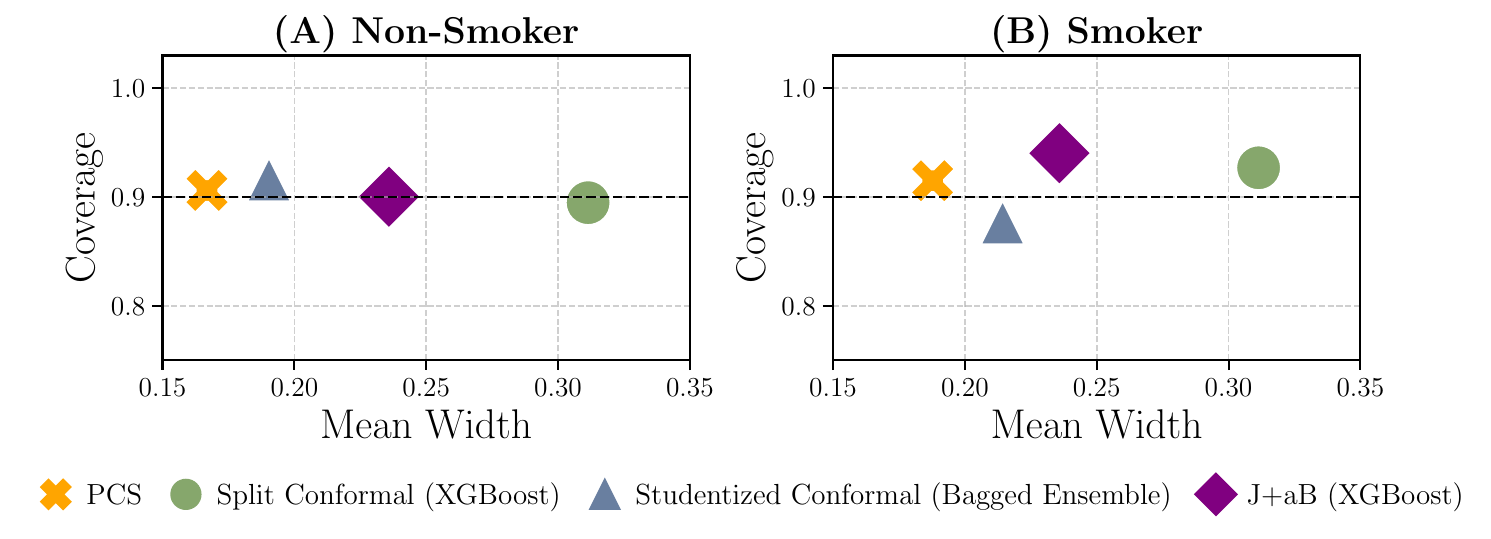}
    \caption{Coverage and width for PCS, and conformal regression approaches on subgroups in the Insurance dataset \cite{ali2020pycaret}. Panels (A) and (B) demonstrate performance for non-smokers and smokers respectively. PCS adapts its width to maintain coverage, while other methods fail to achieve desired coverage or produce larger intervals. }
    \label{fig:regression_subgroup_insurance}
\end{figure}

\paragraph{Energy \cite{tsanas2012energy}} The task is to predict heating load requirements for buildings.
We form subgroups based on roof area of the house. 
PCS adapts its width to maintain coverage across subgroups. 
All conformal methods fail to achieve coverage in the subgroup with smaller roof areas.

\begin{figure}[H]
    \centering
    \includegraphics[width=0.8\textwidth]{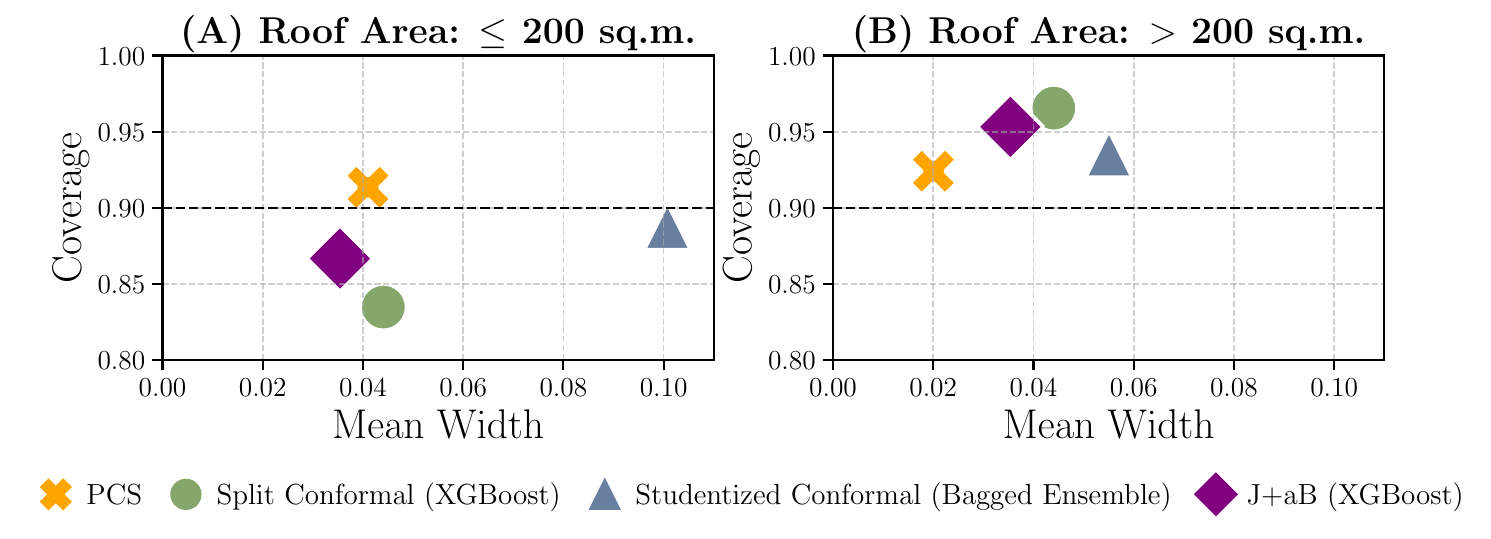}
    \caption{Coverage and width for PCS, and conformal regression approaches on subgroups in the Energy Efficiency dataset \cite{tsanas2012energy}. Panels (A) and (B) demonstrate performance on subgroups formed by roof area of the house. PCS adapts its width to maintain coverage, while other methods fail to achieve desired coverage or produce larger intervals.}
    \label{fig:regression_subgroup_energy}
\end{figure}

\paragraph{Airfoil \cite{brooks1989airfoil}} The task is to predict sound pressure levels of airfoil blades. 
We form subgroups based on the frequency of the sound of the airfoil.
PCS adapts its width to maintain coverage across subgroups and produces matching or shorter intervals than all conformal methods.
Studentized Conformal (Bagged Ensemble) slightly under-covers in the large-frequency subgroup, while the other conformal methods achieve desired coverage.
\begin{figure}[H]
    \centering
    \includegraphics[width=0.8\textwidth]{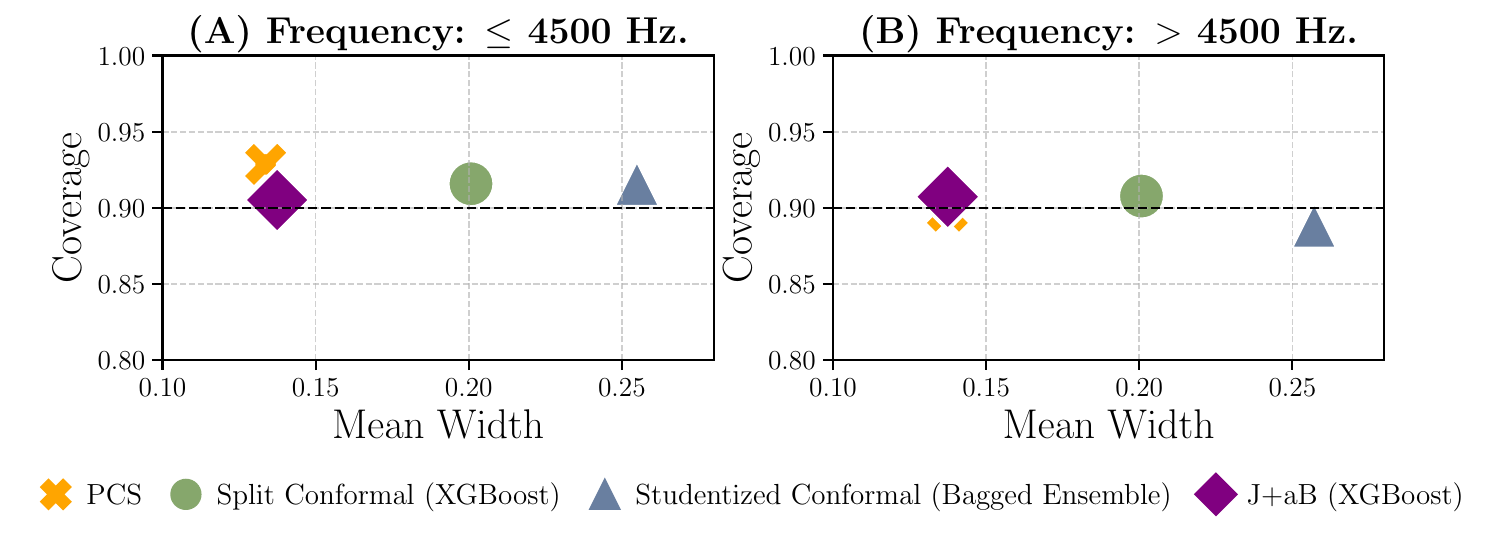}
    \caption{Coverage and width for PCS, and conformal regression approaches on subgroups in the Airfoil dataset \cite{brooks1989airfoil}. Panels (A) and (B) demonstrate performance on subgroups formed by frequency of the sound of the airfoil. PCS adapts its width to maintain coverage, while other methods fail to achieve desired coverage or produce larger intervals.}
    \label{fig:regression_subgroup_airfoil}
\end{figure}

\paragraph{Concrete \cite{yeh1998concrete}} The task is to predict concrete compressive strength.
We form subgroups based on the age of the concrete block.
PCS adapts its width to maintain coverage across subgroups. 
Split Conformal (XGBoost), Studentized Conformal (Bagged Ensemble), and J+aB (XGBoost) slightly under-cover in the small-age subgroup and produce larger intervals in the large-age group.
\begin{figure}[H]
    \centering
    \includegraphics[width=0.8\textwidth]{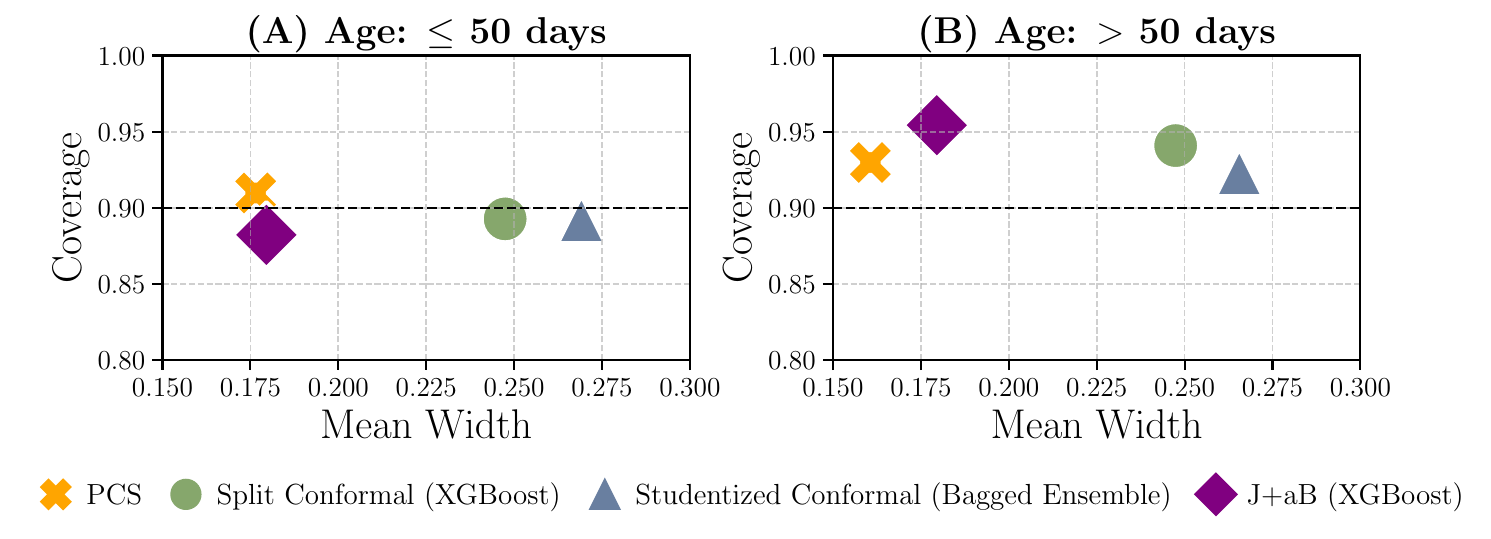}
    \caption{Coverage and width for PCS, and conformal regression approaches on subgroups in the Concrete dataset \cite{yeh1998concrete}. Panels (A) and (B) demonstrate performance on subgroups formed by age of the concrete block. PCS adapts its width to maintain coverage, while other methods fail to achieve desired coverage or produce larger intervals.}
    \label{fig:regression_subgroup_concrete}
\end{figure}

\paragraph{California Housing \cite{kelley1997sparse}} The task is to predict median housing price in Census block groups.
We form subgroups based on the median income of the block group.
PCS adapts its width to maintain coverage across subgroups. 
All conformal methods under-cover in the large-income subgroup.

\begin{figure}[H]
    \centering
    \includegraphics[width=0.8\textwidth]{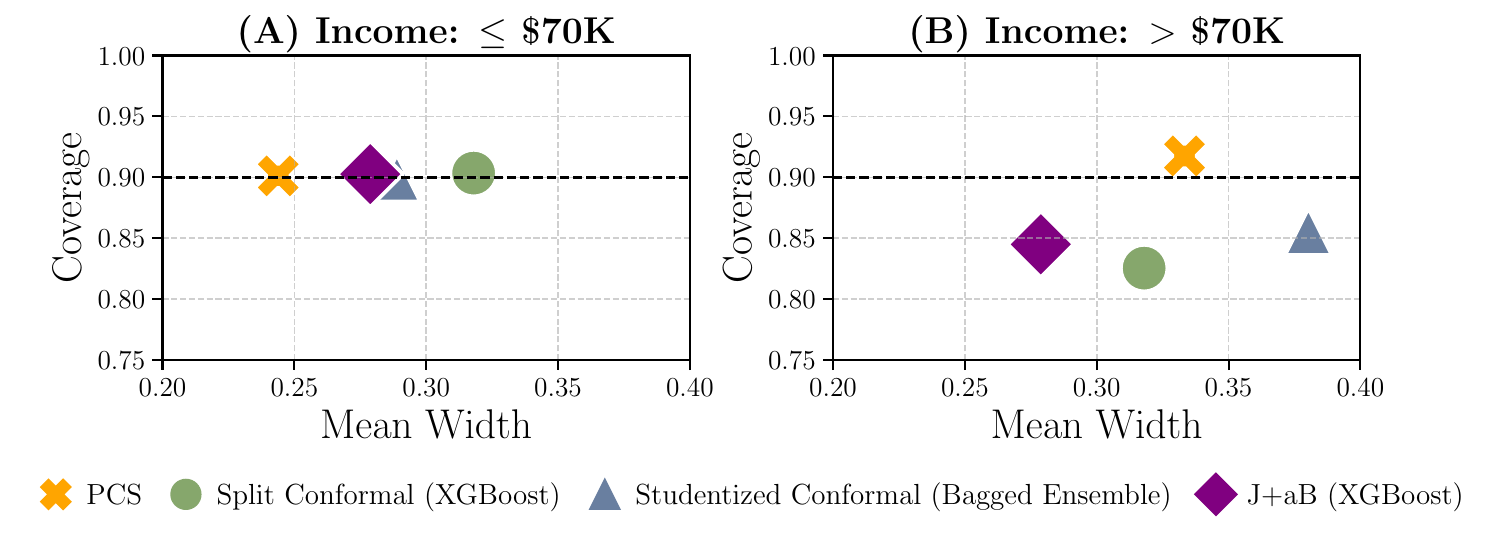}
    \caption{Coverage and width for PCS, and conformal regression approaches on subgroups in the CA Housing dataset \cite{kelley1997sparse}. Panels (A) and (B) demonstrate performance on subgroups formed by median income of the block group. PCS adapts its width to maintain coverage, while other methods fail to achieve desired coverage or produce larger intervals.}
    \label{fig:regression_subgroup_ca}
\end{figure}

\paragraph{Powerplant \cite{tfekci2014combined}} The task is to predict electrical energy output of powerplants.
We form subgroups based on the ambient temperature of the powerplant.
PCS slightly under-covers in the high-temperature subgroup.
Split Conformal (XGBoost), Studentized Conformal (Bagged Ensemble), and J+aB (XGBoost) maintain coverage across subgroups.

\begin{figure}[H]
    \centering
    \includegraphics[width=0.8\textwidth]{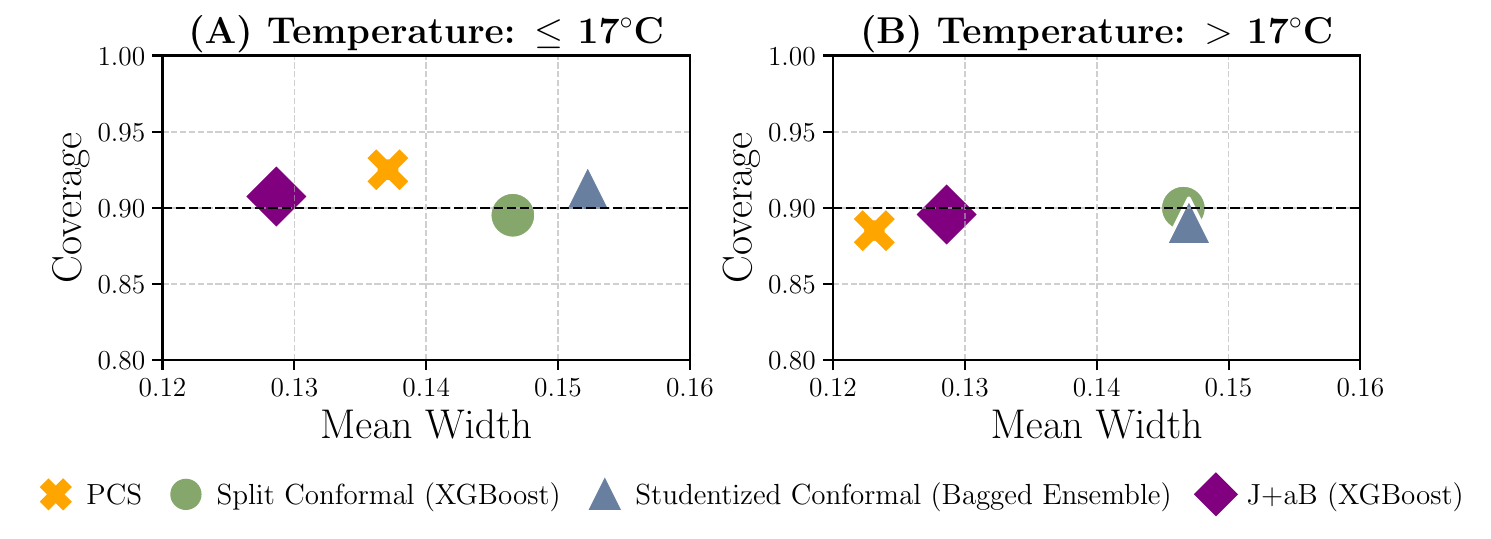}
    \caption{Coverage and width for PCS, and conformal regression approaches on subgroups in the Powerplant dataset \cite{tfekci2014combined}. Panels (A) and (B) demonstrate performance on subgroups formed by ambient temperature of the powerplant. PCS slightly under-covers in the high-temperature subgroup. J+aB (XGBoost) maintains coverage. Other methods produce larger intervals.}
    \label{fig:regression_subgroup_powerplant}
\end{figure}

\subsection{Comparison to PCS Ch.13 of Yu and Barter (2024)} 
\label{supp:reg_comparison_pcs}
We report coverage and mean width of PCS and PCS Ch.13 across our 17 datasets. 
Both methods achieve desired coverage in all datasets.
PCS produces interval widths equal to or smaller than those of PCS Ch.13.
\begin{table}[H]
    \centering
    \begin{tabular}{lrrrr}
\toprule
Method & \multicolumn{2}{c}{PCS} & \multicolumn{2}{c}{PCS Ch.13} \\
 & Coverage & Mean Width & Coverage & Mean Width \\
\midrule
Naval          & 0.905 & \textbf{0.001} & 0.904 & 0.001 \\
Energy         & 0.920 & \textbf{0.030} & 0.918 & 0.035 \\
Computer       & 0.902 & \textbf{0.066} & 0.897 & 0.067 \\
Miami          & 0.903 & \textbf{0.072} & 0.900 & 0.072 \\
Diamond        & 0.902 & \textbf{0.081} & 0.899 & 0.082 \\
Insurance      & 0.910 & \textbf{0.173} & 0.912 & 0.179 \\
Elevator       & 0.904 & \textbf{0.118} & 0.900 & 0.118 \\
Airfoil        & 0.917 & \textbf{0.133} & 0.907 & 0.159 \\
Powerplant     & 0.902 & \textbf{0.129} & 0.894 & 0.132 \\
Concrete       & 0.918 & \textbf{0.171} & 0.917 & 0.209 \\
Debutanizer    & 0.906 & \textbf{0.187} & 0.901 & 0.206 \\
Superconductor & 0.903 & \textbf{0.167} & 0.901 & 0.174 \\
Parkinsons     & 0.905 & \textbf{0.201} & 0.898 & 0.206 \\
CA             & 0.904 & \textbf{0.248} & 0.900 & 0.252 \\
Kin8nm         & 0.905 & \textbf{0.256} & 0.902 & 0.267 \\
QSAR           & 0.903 & \textbf{0.302} & 0.895 & 0.311 \\
Protein        & 0.902 & \textbf{0.491} & 0.899 & 0.502 \\
\bottomrule
\end{tabular}

    \caption{Coverage and mean width of PCS and PCS Ch.13 across our 17 datasets. Both methods achieve desired coverage, while PCS produces equal or smaller interval width than PCS Ch.13.}
    \label{tab:pcs-vs-pcs13}
\end{table}

\subsection{Different Train-Calibration Splits}
\label{supp:reg_different_split}
As suggested in \cite{lei2018distribution}, it may be beneficial for conformal methods to use a higher proportion of the training data to train the predictive algorithm than to calibrate the prediction intervals.
We thus repeat the experiment in \cref{sec:PCS_reg} using 75\% of the training data to train the predictive algorithm, and the other 25\% for calibration.
PCS achieves a smaller but still significant improvement over split and Studentized conformal in most datasets (\cref{fig:regression_7525_width_comparison}), as compared to \cref{fig:regression_width_comparison}.

\begin{figure*}[htbp]
    \centering
    \includegraphics[width=0.9\textwidth]{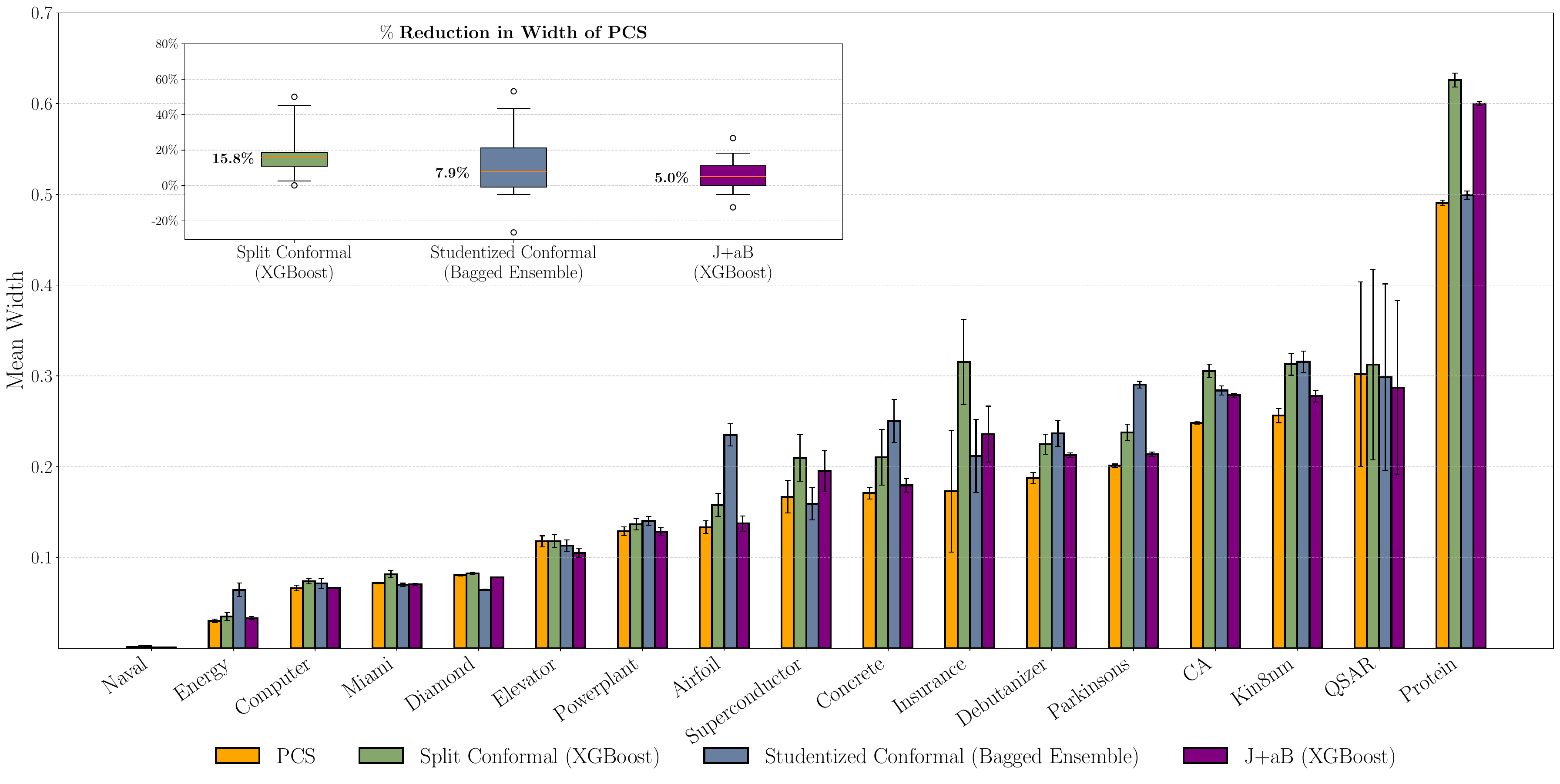}
    \caption{Comparison of PCS against Split conformal (XGBoost), Studentized conformal (Bagged Ensemble), and J+aB (XGBoost) across 17 datasets. We use a 75/25 train-calibration split for conformal methods. PCS-UQ achieves a smaller yet still significant reduction of interval width on average.}
    \label{fig:regression_7525_width_comparison}
\end{figure*}

\subsection{Computation Time Analysis}
One potential concern of using a bootstrap-based method is the time it takes to train the bootstrap models, especially for large training sets.
Both PCS-UQ and J+aB rely on such bootstrap procedures.
We thus conduct an analysis on computation time.
In this experiment, we train PCS-UQ and J+aB on the 17 regression datasets using XGBoost.
Note that this means we use a single fixed predictive algorithm for PCS-UQ and do not conduct prediction-check.
Additionally, we modify the bootstrap step in PCS-UQ to improve computational efficiency.
In the PCS framework \cite{yu2024veridical}, the purpose of the bootstrap is to perturb the dataset, thus creating pseudo-universes represented by the perturbed datasets.
Therefore, any reasonable method of perturbation is allowed in this framework.
We thus replace the bootstrap step in PCS-UQ with a half-subsampling procedure that samples half of the training set without replacement.

We present the results in \cref{tab:pcs-downsample}.
Using half subsampling instead of bootstrapping significantly reduces training time.
As subsampling is also allowed in the J+aB framework, using this sampling scheme can reduce the training time for both PCS-UQ and J+aB.

Additionally, we observe some long inference time for J+aB for two large datasets ($n>40000$).
This is likely due to J+aB performing computation over each training point:
To predict on a new data point $\mathbf{X}$, J+aB iterates through each training point $i$ and computes the aggregate (e.g. average) prediction on $\mathbf{X}$ using bootstrap models that are not trained using point $i$.
The appropriate quantiles of these $n_{\text{train}}$ aggregate predictions (plus/minus the corresponding residuals) are then used to form the prediction interval.
During computation, this procedure should not be time-consuming for typical tabular data, but we indeed observe longer prediction times for larger datasets.

\begin{table}[H]
    \centering

\begin{tabular}{lccccccc}
\toprule
& & \multicolumn{3}{c}{\textbf{Train Time (s)}} & \multicolumn{3}{c}{\textbf{Prediction Time (s)}} \\
\cmidrule(lr){3-5} \cmidrule(lr){6-8}
\textbf{Dataset} & \textbf{n} & \textbf{PCS-UQ} & \textbf{PCS-UQ} & \textbf{J+aB} & \textbf{PCS-UQ} & \textbf{PCS-UQ} & \textbf{J+aB} \\
& & & \textbf{(half subsample)} & & & \textbf{(half subsample)} & \\
\midrule
Energy              & 768    & 63  & 53  & 49  & 1  & 1  & 1  \\
Concrete            & 1030  & 79  & 72  & 72  & 1  & 1  & 1  \\
Insurance           & 1338  & 59  & 52  & 52  & 1  & 1  & 1  \\
Airfoil             & 1503  & 52  & 48  & 46  & 1  & 1  & 1  \\
Debutanizer         & 2394  & 110 & 103 & 102 & 1  & 1  & 1  \\
QSAR                & 5742  & 281 & 206 & 275 & 3  & 3  & 3  \\
Parkinsons          & 5875  & 209 & 211 & 206 & 2  & 2  & 2  \\
Kin8nm              & 8192  & 162 & 149 & 153 & 2  & 2  & 3  \\
Computer            & 8192  & 232 & 205 & 206 & 2  & 2  & 2  \\
Powerplant          & 9568  & 101 & 97  & 94  & 2  & 2  & 3  \\
Naval               & 11934 & 120 & 89  & 107 & 2  & 2  & 4  \\
Miami               & 13932 & 248 & 219 & 218 & 3  & 3  & 5  \\
Elevator            & 16599 & 165 & 130 & 146 & 4  & 4  & 6  \\
CA Housing          & 20640 & 176 & 156 & 153 & 4  & 4  & 7  \\
Superconductor      & 21263 & 973 & 885 & 952 & 5  & 5  & 9  \\
Protein             & 45730 & 240 & 203 & 223 & 8  & 9  & 28 \\
Diamond             & 53940 & 306 & 228 & 246 & 12 & 11 & 57 \\
\bottomrule
\end{tabular}
    \caption{Train and prediction times of PCS-UQ (bootstrap and half subsampling) and J+aB across our 17 datasets. All methods use XGBoost as the base algorithm.}
    \label{tab:pcs-downsample}
\end{table}
\section{Overview of Uncertainty Quantification Methods for Multi-Class Classification}
\label{sec:overview_classification}

\paragraph{Additional Notation} Because many conformal methods rely on ranks of predicted class probabilities, we introduce the following notation. 
For numbers $a_1, a_2, \dots, a_n$, let $\pi$ be the permutation of the indices that sorts the numbers in descending order. That is, $a_{\pi(1)} > a_{\pi(2)} > \dots > a_{\pi(n)}$. 
We assume that $\pi$ arbitrarily breaks ties.
For classification algorithm $\hat{f}(\cdot; \mathcal{D})$, let $\hat{f}^{(c)}(\cdot; \mathcal{D})$ denote the predicted probability for class $c$.
Lastly, we adopt the convention that if a dataset consists of $C$ classes, the classes are labeled $1,2,\dots,C$.

\subsection{Top-K}
\label{subsec:top_k}

We describe the Top-K procedure from \cite{angelopoulos2021uncertainty}.

\paragraph{Step 1: Data-Splitting and Model Training} Randomly split $\mathcal{D}$ into a training set $\mathcal{D}_{\text{tr}}$, and validation set $\mathcal{D}_{\text{val}}$. Fit algorithm $f$ on training set to obtain fitted model $\hat{f}(\cdot; \mathcal{D}_{\text{tr}})$. 

\paragraph{Step 2: Calibration} For each $(\mathbf{X}_i, Y_i) \in \mathcal{D}_{\text{val}}$, make prediction using $\hat{f}(\cdot; \mathcal{D}_{\text{tr}})$ and obtain conformal score $S_i = j$, where $Y_i = \pi(j)$ and $\hat{f}^{(\pi(1))}(\mathbf{X}_i; \mathcal{D}_{\text{tr}}) > \hat{f}^{(\pi(2))}(\mathbf{X}_i; \mathcal{D}_{\text{tr}}) > \dots > \hat{f}^{(\pi(C))}(\mathbf{X}_i; \mathcal{D}_{\text{tr}})$.
That is, $S_i$ is the rank of the predicted probability for the true class.
Let $q$ be the $1 - \alpha$ quantile of the set $\{S_i: i \in [|\mathcal{D}_{\text{val}}|]\}$.

\paragraph{Step 3: Generate Top-K Prediction Set} For a new test point $\mathbf{X}$, produce the prediction set
\begin{equation*}
    \mathcal{S} = \{\pi(1), \pi(2), \dots, \pi(q)\}
\end{equation*}
where $\hat{f}^{(\pi(1))}(\mathbf{X}; \mathcal{D}_{\text{tr}}) > \hat{f}^{(\pi(2))}(\mathbf{X}; \mathcal{D}_{\text{tr}}) > \dots > \hat{f}^{(\pi(C))}(\mathbf{X}; \mathcal{D}_{\text{tr}})$.

\subsection{Adaptive Prediction Sets}
\label{subsec:aps}

We describe the Adaptive Prediction Sets (APS) procedure from \cite{romano2020classification}.

\paragraph{Step 1: Data-Splitting and Model Training} Randomly split $\mathcal{D}$ into a training set $\mathcal{D}_{\text{tr}}$, and validation set $\mathcal{D}_{\text{val}}$. Fit algorithm $f$ on training set to obtain fitted model $\hat{f}(\cdot; \mathcal{D}_{\text{tr}})$. 

\paragraph{Step 2: Calibration} For each $(\mathbf{X}_i, Y_i) \in \mathcal{D}_{\text{val}}$, predict using $\hat{f}(\cdot; \mathcal{D}_{\text{tr}})$ and obtain conformal score $S_i = \sum_{j=1}^t \hat{f}^{(\pi(j))}(\mathbf{X}_i; \mathcal{D}_{\text{tr}})$, where $\pi(t) = Y_i$.
Let $q$ be the $1 - \alpha$ quantile of the set $\{S_i: i \in [|\mathcal{D}_{\text{val}}|]\}$.

\paragraph{Step 3: Generate APS Prediction Set} For a new test point $\mathbf{X}$, produce the prediction set
\begin{equation*}
    \mathcal{S} = \{\pi(1), \pi(2), \dots, \pi(v)\},\,\,\, \text{where } v = \min \bigg\{t: \sum_{j=1}^t \hat{f}^{(\pi(j))}(\mathbf{X}; \mathcal{D}_{\text{tr}}) \geq q \bigg\}
\end{equation*}

\subsection{Regularized Adaptive Prediction Sets}
\label{subsec:raps}

We describe the Regularized Adaptive Prediction Sets (RAPS) procedure from \cite{angelopoulos2021uncertainty}.

\paragraph{Step 1: Data-Splitting and Model Training} Randomly split $\mathcal{D}$ into a training set $\mathcal{D}_{\text{tr}}$, a tuning set $\mathcal{D}_{\text{tune}}$, and validation set $\mathcal{D}_{\text{val}}$. Fit algorithm $f$ on training set to obtain fitted model $\hat{f}(\cdot; \mathcal{D}_{\text{tr}})$. 

\paragraph{Step 2: Hyperparameter Tuning} The procedure has 2 hyperparameters: $t_{\text{reg}}$ and $\lambda$.
To tune $t_{\text{reg}}$, we perform Step 2 of the Top-K procedure (\ref{subsec:top_k}) on $\mathcal{D}_{\text{tune}}$ and set $t_{\text{reg}} = q$.
To tune $\lambda$, we perform a grid search with candidate values:
using the tuned $t_{\text{reg}}$ and a candidate $\lambda$, we proceed to Steps 3 and 4 using $\mathcal{D}_{\text{tune}}$.
We pick the $\lambda$ that produces the smallest average prediction set size.

\paragraph{Step 3: Calibration} For each $(\mathbf{X}_i, Y_i) \in \mathcal{D}_{\text{val}}$, predict using $\hat{f}(\cdot; \mathcal{D}_{\text{tr}})$ and obtain conformal score
\begin{equation*}
    S_i = \sum_{j=1}^t \hat{f}^{(\pi(j))}(\mathbf{X}_i; \mathcal{D}_{\text{tr}}) + \lambda (t - t_{\text{reg}})^+
\end{equation*}
where $\pi(t) = Y_i$ and $(\cdot)^{+}$ denotes the positive part. 
Let $q$ be the $1 - \alpha$ quantile of the set $\{S_i: i \in [|\mathcal{D}_{\text{val}}|]\}$.

\paragraph{Step 4: Generate RAPS Prediction Set} For a new test point $\mathbf{X}$, produce the prediction set
\begin{equation*}
    \mathcal{S} = \{\pi(1), \pi(2), \dots, \pi(v)\},\,\,\, \text{where } v = \min \bigg\{t: \sum_{j=1}^t \hat{f}^{(\pi(j))}(\mathbf{X}; \mathcal{D}_{\text{tr}}) + \lambda (t - t_{\text{reg}})^+ \geq q \bigg\}
\end{equation*}

\section{Additional Classification Results}

We report coverage for the best-performing method (as measured by average width) across our $6$ tabular classification datasets. 
All methods achieve desired coverage.

\begin{table}[H]
    \centering

\begin{tabular}{lrrrr}
\toprule
Method & RAPS & Top-k & APS & PCS \\
\midrule
Dataset &  &  &  &  \\
\midrule
Chess       & 0.901  & 0.912  & 0.901  & 0.900 \\
Cover Type  & 0.904  & 0.906  & 0.900  & 0.899 \\
Dionis      & 0.943  & 0.904  & 0.909  & 0.901 \\
Isolet      & 0.983  & 0.928  & 0.938  & 0.965 \\
Language    & 0.902  & 0.913  & 0.910  & 0.907 \\
Yeast       & 0.896  & 0.920  & 0.901  & 0.901 \\
\bottomrule
\end{tabular}
    \caption{Coverage for PCS, and best-performing conformal methods for our multi-class classification experiments. }
    \label{tab:coverage_classification}
\end{table}

\subsection{Description of Classification Datasets} We describe the context of the datasets used in our classification experiments.

\paragraph{Language \cite{collins2003collins}} The dataset contains quantitative measurements of bodies of literature in English. Example features include frequency of first-person usage and frequency of past-tense usage. The goal is to predict the genre of the text, out of 30 potential genres such as fiction, memoir, and poetry.

\paragraph{Yeast \cite{horton1996probabilistic}} The dataset contains measurements on amino acid sequences of yeast proteins. Example features include discriminant analysis output of amino acid sequences, and  nuclear localization consensus patterns. The goal is to predict the type of the yeast protein, out of 10 potential types such as cytoskeletal, nuclear and mitochondrial.

\paragraph{Isolet \cite{cole1991isolet}} The dataset contains characteristics of voice in recordings that each contain a single English letter. Example features include spectral coefficients, contour features, and sonorant features. The goal is to predict the letter pronounced, out of the 26 English letters.

\paragraph{Cover Type \cite{blackard1998covertype}} The dataset contains cartographic variables of regions in the Roosevelt National Forest. Each observation represents a 30 by 30 meter cell. Example features include elevation, slope, and distance to nearest road. The goal is to predict the cover-soil type, out of the 100 potential types such as spruce-sand and pine-clay.

\paragraph{Chess \cite{alcalafdez2011keel}} The dataset contains positions of both kings and the white rook in chess endgames. Example features include the row and column of the three pieces. The goal is to predict the number of optimal moves until white wins the game; if the number of moves is more than 16, the game ends in a draw. Hence the potential classes are $0,\dots,16$ and ``draw''.

\paragraph{Dionis \cite{guyon2019analysis}} The dataset is anonymized from a machine learning challenge. The dataset has 60 numerical features and 355 classes. No further context is available.
\section{Additional Deep-Learning Results}
\label{supp:dl}

We report coverage for the UQ methods across our 3 deep learning datasets. All methods achieve desired coverage.

\begin{table}[h]
\centering
\small
\setlength{\tabcolsep}{5pt}
\renewcommand{\arraystretch}{1.2}
\begin{tabular}{llccc}
\toprule
& \textbf{Method} & \textbf{CIFAR-100} & \textbf{Caltech Birds} & \textbf{ImageNet-Small} \\
\midrule
& \textcolor{blue}{APS} & 0.906 & 0.914 & 0.906 \\
& \textcolor{cyan}{RAPS} & 0.908 & 0.911 & 0.905 \\
& TopK & 0.923 & 0.916 & 0.923 \\
\midrule
\multirow{3}{*}{\rotatebox[origin=c]{90}{\textcolor{orange}{PCS}}}
& Original & 0.905 & 0.924 & 0.904 \\
& Dropout & 0.913 & 0.921 & 0.906 \\
& Noise & 0.911 & 0.919 & 0.905 \\
\bottomrule
\end{tabular}
\caption{Coverage across multiple deep learning datasets. All methods achieve target coverage.}
\label{tab:pcs_dl_coverage}
\end{table}

\subsection{Description of Datasets}

\paragraph{CIFAR-100 \cite{krizhevsky2009learning}} This dataset consists of $60000$ $32\times32$ natural-colored images in 100 classes, each containing 600 images per class. There are 50000 training images and 10000 test images.

\paragraph{ImageNet-Small \cite{imagenet_cvpr09}} ImageNet-Small contains 100000 natural images of 200 classes (500 for each class) downsized to $64\times64$ colored images. Each class has 500 training images, 50 validation images and 50 test images.

\paragraph{Caltech-UCSD Birds \cite{welinder2010caltech}} This is an image dataset with photos of 200 mostly North American bird species.

\section{Theoretical Results}
\label{supp:theory}
We describe the modified PCS procedure for regression, and then establish our formal theoretical results as follows.

\subsection{Modified PCS Procedure}
\label{supp:modified_pcs}
We use the notation established in \cref{sec:PCS_reg}.
\begin{enumerate}[leftmargin=*]
    \item Split $\mathcal{D}$ into $\mathcal{D}_{\text{tr}}$,  $\mathcal{D}_{\text{val}}$, and $\mathcal{D}_{\text{cal}}$, and conduct the prediction check as described in step 1 of \cref{sec:PCS_reg} on $\mathcal{D}_{\text{val}}$ to obtain screened algorithms $f_{1}, \ldots f_{k}$.
    \item Bootstrap the \emph{training} dataset $B$ times to obtain bootstrapped samples $\mathcal{D}_{\text{tr}}^{(1)} \ldots \mathcal{D}_{\text{tr}}^{(B)}$. Fit all algorithms chosen in the previous step on every bootstrapped dataset $\mathcal{D}_{\text{tr}}^{(b)}$ to obtain bootstrapped models $\{\hat{f}_j(; \mathcal{D}_{tr}^{(b)}), j \in [k], b \in [B]\}$. 
    
    \vspace{2mm}
    To define the calibration procedure, we introduce some necessary notation.
    For a given point $\mathbf{X}$, form a prediction set $\mathcal{P} = \{\hat{f}_j(\mathbf{X}; \mathcal{D}_{\text{tr}}^{(b)}); j \in [k], b \in B\}$.
    Further, let $q_{\beta}(S)$ be the $\beta$ quantile for a set $S$ and let
    \[
    l_{\alpha}(\mathbf{X}) := q_{\alpha/2}(\mathcal{P}),\quad u_\alpha(\mathbf{X}) := q_{1-\alpha/2}(\mathcal{P}),\quad m(\mathbf{X}) := q_{0.5}(\mathcal{P}).
    \]
     Next, we define score function for $\mathbf{X}$ as follows
    \begin{equation}
    \label{eq:PCS_modified_calib}
    S(\mathbf{X},Y) = \min\left\{\gamma: Y \in \left[
    m(\mathbf{X})-\gamma \times (m(\mathbf{X}) - l_{\alpha}(\mathbf{X})) , 
    m(\mathbf{X})+\gamma \times (u_\alpha(\mathbf{X}) - m(\mathbf{X}))\right]\right\}.
    \end{equation}
    Moreover, let $S_{\text{cal}} = \{S(\mathbf{X}_i,Y_i): i \in D_{\text{cal}}\}$.
    Let $\hat{\gamma}_{\alpha}$ be the $\lceil (1-\alpha) (|D_{cal}|+1) \rceil$-th smallest element of $S_{\text{cal}}$
    \vspace{2mm}
    
    \item For a test point $\mathbf{X}_*$, define the PCS prediction set as 
    \begin{equation}
\label{eq:modified_pcs_interval}
    \hat{C}_{\text{PCS}}(\mathbf{X}_*) = \left[m(\mathbf{X}_*)-\hat{\gamma}_{\alpha} \times (m(\mathbf{X}_*) - l_{\alpha}(\mathbf{X}_*)) , 
    m(\mathbf{X}_*)+\hat{\gamma}_{\alpha} \times (u_\alpha(\mathbf{X}_*) - m(\mathbf{X}_*))\right],
\end{equation}
   
\end{enumerate}

Next, we provide our formal theoretical result and its proof. 

\begin{theorem} For a test point ($\mathbf{X}_*,Y)$, assume $\mathcal{D_\textbf{cal}} \cup (\mathbf{X}_*,Y)$ is exchangeable. 
For given $\alpha \in (0,1)$, the PCS prediction interval \eqref{eq:modified_pcs_interval} satisfies 
\begin{equation*}
    \mathbb{P}_{(\mathbf{X}_i, Y_i) \in \mathcal{D}_{\text{tr}} ~\cup \mathcal{D}_{\text{val}}}\left(Y \in \hat{C}_{\text{PCS}}(\mathbf{X}_{*})\right) \geq 1- \alpha  
\end{equation*}

\label{thm:pcs_coverage}

\end{theorem}

\begin{proof} The proof follows from the fact that we can rewrite the modified PCS prediction set \eqref{eq:modified_pcs_interval} as follows
\begin{equation*}
    \hat{C}_{\text{PCS}}(\mathbf{X_*}) = \{y: S(\mathbf{X_*},y) \leq \hat{\gamma}_{\alpha}\}.
\end{equation*}
Given this observation, we use the result of \cite{shafer2007tutorialconformalprediction} that any prediction set with a prefitted score function defined with the form above is guaranteed to have $1-\alpha$ coverage.  
\end{proof}

\end{document}